\documentclass{article}

\usepackage[preprint, nonatbib]{nips_2018}

\usepackage[utf8]{inputenc} 
\usepackage[T1]{fontenc}    
\usepackage{hyperref}       
\usepackage{url}            
\usepackage{amsfonts}       
\usepackage{nicefrac}       
\usepackage{microtype}      
\usepackage{overpic}
\usepackage{multirow}
\usepackage{centernot}
\usepackage{algorithmic}
\usepackage{algorithm}
\usepackage{enumerate}
\usepackage{mathtools}
\usepackage{stmaryrd}
\usepackage{epstopdf}
\usepackage{amsthm}
\usepackage{graphicx} 
\usepackage{amsmath}  
\usepackage{amsfonts} %
\usepackage{amssymb}  %
\usepackage[disable]{todonotes}
\usepackage{booktabs}
\usepackage{wrapfig,lipsum}

\newtheorem{theorem}{Theorem}
\newtheorem{lemma}[theorem]{Lemma}

\newtheorem{corollary}[theorem]{Corollary}

\newtheorem{example}[theorem]{Example}

\DeclareMathOperator*{\E}{{\mathbb{E}}}
\DeclareMathOperator*{\M}{{\mathbb{M}}}
\DeclareMathOperator*{\var}{\mathrm{Var}}

\DeclareMathOperator*{\poly}{poly}
\DeclareMathOperator*{\dist}{d}
\DeclareMathOperator*{\supp}{supp}

\newcommand{\secprespace}[0]{\vspace{-1mm}}
\newcommand{\secspace}[0]{\vspace{-1.5mm}}
\newcommand{\paraspace}[0]{\vspace{-3mm}}
\newcommand{\figref}[1]{Fig.~\ref{fig:#1}}
\newcommand{\tabref}[1]{Table~\ref{tab:#1}}
\newcommand{\secref}[1]{Sec.~\ref{sec:#1}}
\newcommand{\thmref}[1]{Thm.~\ref{thm:#1}}
\newcommand{\appref}[1]{Appx.~\ref{sec:#1}}
\newcommand{\txt}[1]{\textrm{#1}}

\newlength\savedwidth
\newcommand\whline[1]{\noalign{\global\savedwidth\arrayrulewidth
                               \global\arrayrulewidth #1} %
                      \hline
                      \noalign{\global\arrayrulewidth\savedwidth}}

\title{BourGAN: Generative Networks with Metric Embeddings}

\author{
  Chang Xiao \qquad Peilin Zhong\qquad Changxi Zheng\\
  Columbia University\\
  \texttt{\{chang, peilin, cxz\}@cs.columbia.edu} \\
}

\begin{document}

\maketitle

\begin{abstract}
This paper addresses the mode collapse for generative adversarial networks (GANs).
We view modes as a \emph{geometric} structure of data distribution in a metric space.
Under this geometric lens, we embed subsamples of the dataset from an arbitrary metric space into the $\ell_2$
space, while preserving their pairwise distance distribution.
Not only does this metric embedding determine the dimensionality of the latent space automatically, it also
enables us to construct a \emph{mixture of Gaussians} to draw latent space random vectors.
We use the Gaussian mixture model in tandem with a simple augmentation of the objective function to
train GANs. Every major step of our method is supported by theoretical analysis, and our experiments
on real and synthetic data confirm that the generator is able to produce samples spreading over most of the modes
while avoiding unwanted samples, outperforming several recent GAN variants on a number of metrics and offering new features.
\end{abstract}

\secprespace
\section{Introduction}
\secspace
In unsupervised learning, Generative Adversarial Networks (GANs)~\cite{goodfellow2014generative} is
by far one of the most widely used methods for training deep generative models.
However, difficulties of optimizing GANs have also been well observed
~\cite{nowozin2016f, salimans2016improved, arjovsky2017wasserstein, dumoulin2016adversarially, donahue2016adversarial,
reed2016generative, radford2015unsupervised}.
One of the most prominent issues is \emph{mode collapse}, a phenomenon in which a GAN,
after learning from a data distribution of multiple modes, generates samples landed
only in a subset of the modes. In other words, the generated samples lack the diversity
as shown in the real dataset, yielding a much lower entropy distribution.

We approach this challenge by questioning two fundamental properties
of GANs. i) We question the commonly used multivariate Gaussian that generates
random vectors for the generator network.
We show that in the presence of separated modes, drawing random vectors from
a single Gaussian may lead to arbitrarily large gradients of the generator, 
and a better choice is by using a \emph{mixture of Gaussians}.
ii) We consider the \emph{geometric} interpretation of modes, and
argue that the modes of a data distribution should be viewed under a specific distance
metric of data items -- different metrics may lead to different distributions of modes,
and a proper metric can result in interpretable modes. 
From this vantage point, we address the problem of mode collapse in a general metric space.
To our knowledge, despite the recent attempts of addressing mode
collapse~\cite{salimans2016improved,metz2016unrolled,srivastava2017veegan,
donahue2016adversarial, lin2017pacgan, che2016mode}, both properties remain unexamined.

\paraspace
\paragraph{Technical contributions.}
We introduce \emph{BourGAN}, an enhancement of GANs to avoid mode collapse in
any metric space.  In stark contrast to all existing mode collapse solutions,
BourGAN draws random vectors from a Gaussian mixture in a low-dimensional
latent space.  The Gaussian mixture is constructed to mirror the mode structure
of the provided dataset under a given distance metric.  We derive the
construction algorithm from metric embedding theory, namely the Bourgain
Theorem~\cite{b85}.
Not only is using metric embeddings theoretically sound (as we will show), it
also brings significant advantages in practice.
Metric embeddings enable us to retain the mode structure in the $\ell_2$ latent space 
despite the metric used to measure modes in the dataset.
In turn, the Gaussian mixture sampling in the latent space eases the
optimization of GANs, and unlike existing GANs that assume a user-specified
dimensionality of the latent space, our method automatically decides the
dimensionality of the latent space from the provided dataset.

To exploit the constructed Gaussian mixture for addressing mode collapse,
we propose a simple extension to the GAN objective that encourages
the pairwise $\ell_2$ distance of latent-space random vectors to match
the distance of the generated data samples in the metric space.
That is, the geometric structure of the Gaussian mixture is respected in the generated samples.
Through a series of (nontrivial) theoretical analyses, we show that
if BourGAN is fully optimized, 
the logarithmic pairwise distance distribution of its generated samples
closely match the logarithmic pairwise distance distribution of the real data items.
In practice, this implies that mode collapse is averted.

We demonstrate the efficacy of our method on both synthetic and real datasets. 
We show that our method outperforms several recent GAN variants in terms of generated data diversity.
In particular, our method is robust to handle data distributions with multiple separated modes -- challenging 
situations where all existing GANs that we have experimented with produce unwanted samples (ones that are not in any modes),
whereas our method is able to generate samples spreading over all modes 
while avoiding unwanted samples.\todo{emphasize experiments in practice rather than theory}













\secprespace
\section{Related Work}
\secspace
\paragraph{GANs and variants.}
The main goal of generative models in unsupervised learning is to produce
samples that follow an unknown distribution $\mathcal{X}$, by learning from
a set of unlabelled data items $\{x_i\}_{i=1}^n$ drawn from $\mathcal{X}$.
In recent years, Generative Adversarial Networks (GANs)~\cite{goodfellow2014generative}
have attracted tremendous attention for training generative models.
A GAN uses a neural network, called generator $G$, to map a low-dimensional
latent-space vector $z\in\mathbb{R}^d$, drawn from a standard distribution
$\mathcal{Z}$ (e.g., a Gaussian or uniform distribution), to generate data items in a space of
interest such as natural images and text.
The generator $G$ is trained in tandem with another neural network, called the discriminator $D$,
by solving a minmax optimization with the following objective.
\begin{equation} \label{eq:ganloss}
    L_{\txt{gan}}(G, D) =  \mathbb{E}_{x \sim \mathcal{X}}\left[\log{D(x)}\right] + \mathbb{E}_{z \sim \mathcal{Z}}\left[\log(1-D(G(z)))\right].
\end{equation}
This objective is minimized over $G$ and maximized over $D$.
Initially, GANs are demonstrated to generate locally appreciable but
globally incoherent images. Since then, they have been actively improving.
For example, DCGAN~\cite{radford2015unsupervised} proposes a class of empirically designed
network architectures that improve the naturalness of generated images.
By extending the objective~\eqref{eq:ganloss}, InfoGAN~\cite{chen2016infogan} is able
to learn interpretable representations in latent space, Conditional GAN ~\cite{mirza2014conditional} 
can produce more realistic results by using additional supervised label.
Several later variants have applied GANs to a wide array of tasks ~\cite{bora2018ambientgan,bora2017compressed} 
such as image-style transfer~\cite{isola2017image,zhu2017unpaired},
super-resolution~\cite{ledig2016photo}, image
manipulation~\cite{zhu2016generative}, video
synthesis~\cite{vondrick2016generating}, and 3D-shape synthesis
\cite{wu2016learning}, to name a few.

\begin{figure}
    \centering
    \begin{overpic}[width=1.0\textwidth]{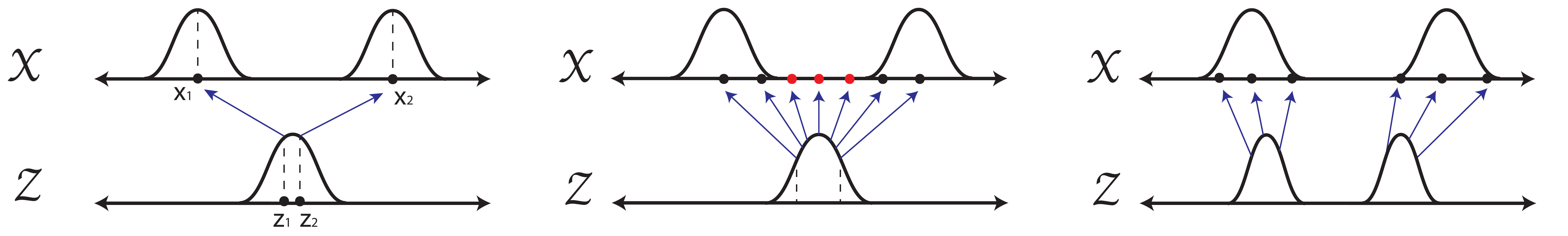}
    \put(17.3,-1.7){(a)}
    \put(50.9,-1.7){(b)}
    \put(84.0,-1.7){(c)}
    \end{overpic}
    \vspace{-2mm}
    \caption{\textbf{Multi-mode challenge.}
We train a generator $G$ that maps a latent-space distribution $\mathcal{Z}$
to the data distribution $\mathcal{X}$ with two modes.
\textbf{(a)} Suppose $\mathcal{Z}$ is a Gaussian, and $G$ can fit both modes.
If we draw two i.i.d. samples $z_1,z_2$ from $\mathcal{Z}$, then with at least
a constant probability, $G(z_1)$ is close to the center $x_1$ of the first mode,
and $G(z_2)$ is close to another center $x_2$.
By the Mean Value Theorem, there exists a $z$ between $z_1$ and $z_2$
that has the absolute gradient value, $|G'(z)|=|\frac{x_2-x_1}{z_2-z_1}|$,
which can be arbitrarily large, as $|x_2-x_1|$ can be arbitrarily far.
\textbf{(b)} Since $G$ is Lipschitz continuous, using it to map a Gaussian distribution to
both modes unavoidably results in unwanted samples between the modes (highlighted by the red dots).
\textbf{(c)} Both challenges are resolved if we can construct a mixture of Gaussian
in latent space that captures the same modal structure as in the data distribution.
}\label{fig:gaussian_flaw}
    \vspace{-1.5mm}
\end{figure}

\paraspace
\vspace{1mm}
\paragraph{Addressing difficulties.}
Despite tremendous success, GANs are generally hard to train.
Prior research has aimed to improve the stability of training GANs, mostly by
altering its objective function~\cite{mao2017least,arjovsky2017wasserstein,gulrajani2017improved,zhao2016energy, saatci2017bayesian, arora2017generalization}.
In a different vein, Salimans et al.~\cite{salimans2016improved} proposed a
feature-matching technique to stabilize the training process,
and another line of work~\cite{dumoulin2016adversarially,donahue2016adversarial,larsen2015autoencoding}
uses an additional network
that maps generated samples back to latent vectors to provide feedback
to the generator.

A notable problem of GANs is mode collapse,
which is the focus of this work.
For instance, when trained on ten hand-written digits (using MNIST
dataset)~\cite{lecun1998gradient}, each digit represents a mode of data
distribution, but the generator often fails to produce
a full set of the digits~\cite{gulrajani2017improved}.
Several approaches have been proposed to mitigate mode collapse,
by modifying either
the objective function~\cite{arjovsky2017wasserstein, che2016mode}
or the network architectures~\cite{metz2016unrolled,dumoulin2016adversarially,lin2017pacgan,srivastava2017veegan,karras2017progressive}.
While these methods are evaluated empirically, theoretical understanding
of why and to what extent these methods work is often lacking.
More recently, PacGAN~\cite{lin2017pacgan} introduces a mathematical definition of mode collapse,
which they used to formally analyze their GAN variant.
Very few previous works consider the construction of latent space:
VAE-GAN \cite{larsen2015autoencoding} constructs the
latent space using variational autoencoder, and  GLO~\cite{bojanowski2017optimizing} tries to
optimize both the generator network and latent-space representation using data samples.
Yet, all these methods still draw the latent random vectors from a multivariate Gaussian.

\paraspace
\paragraph{Differences from prior methods.}
%
Our approach differs from prior methods in several important technical aspects. 
Instead of using a standard Gaussian to sample latent space, we propose to use
a Gaussian mixture model constructed using metric embeddings (e.g., see \cite{matouvsek2002embedding,courty2017learning,indyk2004low} 
for metric embeddings in both theoretical and machine learning fronts).
Unlike all previous methods that require the latent-space dimensionality to be
specified \emph{a priori},
our algorithm automatically determines its dimensionality from the
real dataset. Moreover, our method is able to incorporate any distance metric,
allowing the flexibility of using proper metrics for learning interpretable
modes.  In addition to empirical validation, the steps of our method are grounded
by theoretical analysis.

\secprespace
\section{Bourgain Generative Networks}\label{sec:bgn}
\secspace
We now introduce the algorithmic details of BourGAN, starting by describing the rationale behind
the proposed method. The theoretical understanding of our method will be presented in the next section.
\todo{we showed this method works and knows WHY it works}

\paraspace
\paragraph{Rationale and overview.}
We view modes in a dataset as a \emph{geometric structure} embodied
under a specific distance metric.  For example, in the widely tested MNIST
dataset, only two modes emerge under the pixel-wise $\ell_2$ distance
(\figref{pdd_tsne}-left): images for the digit ``1'' are clustered in one mode, while
all other digits are landed in another mode.  In contrast, under the classifier distance
metric (defined in \appref{mnist_app}), it appears that there exist 10 modes each corresponding to
a different digit. Consequently, the modes are interpretable (\figref{pdd_tsne}-right). In this work, we aim to incorporate any
distance metric when addressing mode collapse, leaving the flexibility of choosing a specific metric
to the user.

\begin{figure}
  \centering
  \includegraphics[width=0.85\textwidth]{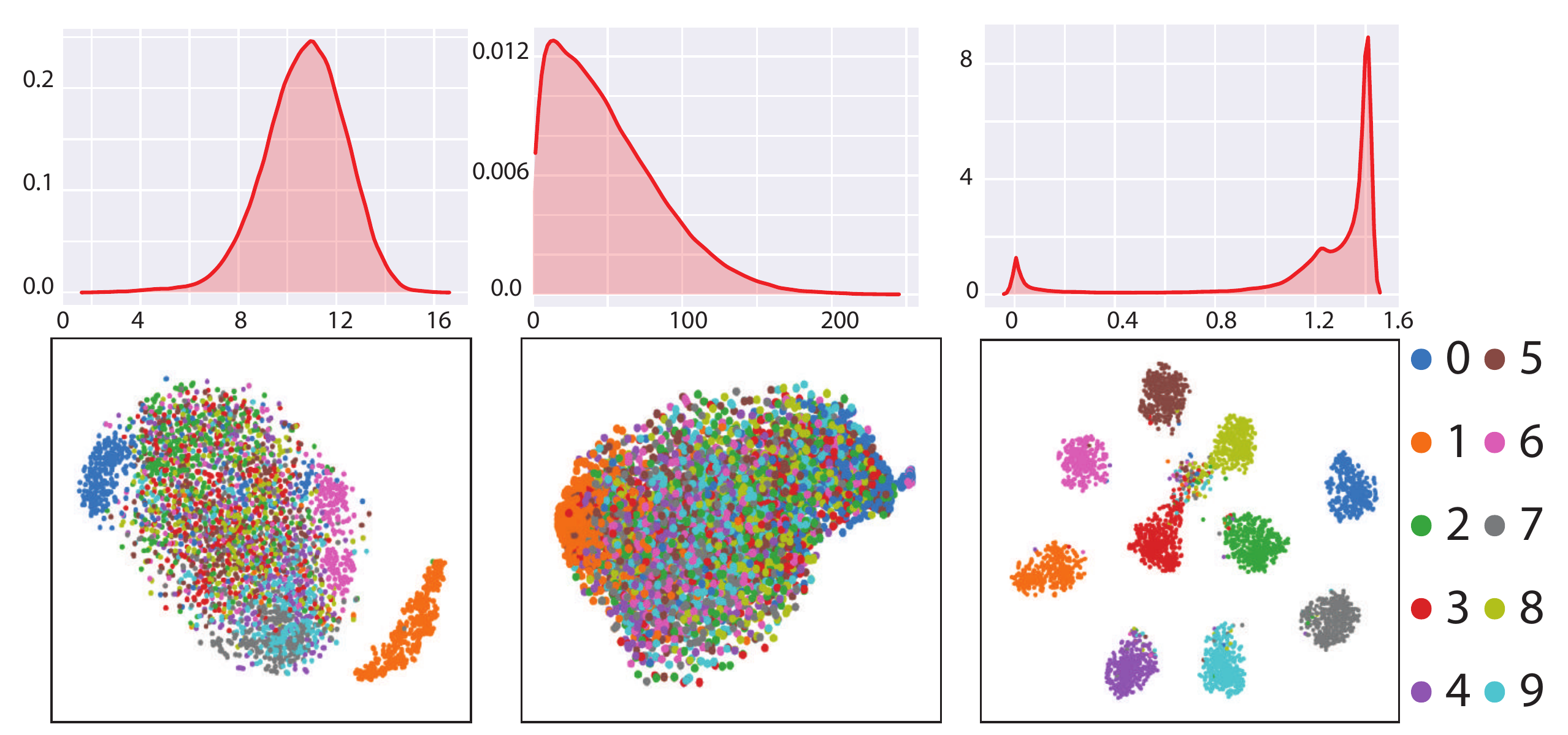}
  \vspace{-1mm}
  \caption{ 
  \textbf{(Top)}  Pairwise distance distribution on MNIST dataset under different distance metrics.
  Left: $\ell_2$ distance, Middle: Earth Mover's distance (EMD) with a quadratic ground
  metric, Right: classifier distance (defined in \appref{mnist_app}).
  Under $\ell_2$ and EMD distances, few separated modes emerges, and the pairwise
  distance distributions resemble a Gaussian. Under the classifier distance,
  the pairwise distance distribution becomes bimodal, indicating that there are separated modes.
  \textbf{(Bottom)} t-SNE visualization \cite{maaten2008visualizing} of
  data items after embedded from their metric space to $\ell_2$ space. Color indicates
  labels of MNIST images (``1''-``9''). When $\ell_2$ distance (left) is used, only two modes
  are identified: digit ``1'' and all others, but classifier distance (right) can group
  data items into 10 individual modes.
  }\label{fig:pdd_tsne}
  \vspace{-1mm}
\end{figure}

When there are multiple separated modes in a data distribution, mapping a Gaussian random variable in latent
space to the data distribution is fundamentally ill-posed.
For example, as illustrated in \figref{gaussian_flaw}-a and \ref{fig:gaussian_flaw}-b,
this mapping imposes arbitrarily large gradients
(at some latent space locations)
in the generator network, and large gradients render the generator unstable to
train, as pointed out by~\cite{nagarajan2017gradient}.

A natural choice is to use a mixture of Gaussians. As long as the Gaussian
mixture is able to mirror the mode structure of the given dataset, 
the problem of mapping it to the data distribution becomes well-posed
(\figref{gaussian_flaw}-c). To this end, our main idea is to use metric
embeddings, one that map data items under any metric to a low-dimensional
$\ell_2$ space with bounded pairwise distance distortion (\secref{gm}).
After the embedding, we construct a Gaussian
mixture in the $\ell_2$ space, regardless of the distance metric for the data
items. In this process, the dimensionality of the latent space is also automatically decided.

Our embedding algorithm, building upon the Bourgain Theorem, requires us
to compute the pairwise distances of data items, resulting in an $O(n^2)$
complexity, where $n$ is the number of data items. When $n$ is large, we
first uniformly subsample $m$ data items from the dataset
to reduce the computational cost of our metric embedding algorithm (\secref{subsample}).
The subsampling step is theoretically sound: we prove that when $m$ is sufficiently large yet still much smaller than $n$,
the geometric structure (i.e., the pairwise distance distribution) of data items is preserved in the subsamples.

Lastly, when training a BourGAN, we encourage the geometric structure embodied
in the latent-space Gaussian mixture to be preserved by the generator network.
Thereby, the mode structure of the dataset is learned by the generator.
This is realized by augmenting GAN's objective to foster the preservation of the
pairwise distance distribution in the training process (\secref{training}). 

\secprespace
\subsection{Metrics of Distance and Distributions}\label{sec:metric}
\secspace
Before delving into our method, we introduce a few theoretical tools to
concretize the geometric structure in a data distribution, paving the way toward understanding
our algorithmic details and subsequent theoretical analysis. In the rest of this paper, we borrow a few notational
conventions from theoretical computer science:
we use $[n]$ to denote the set $\{1,2,\cdots,n\}$, $\mathbb{R}_{\geq 0}$
to denote the set of all non-negative real numbers, and $\log(\cdot)$ to
denote $\log_2(\cdot)$ for short.

\paraspace
\paragraph{Metric space.}
A metric space is described by a pair $(\mathbb{M},\dist)$, where $\mathbb{M}$ is a set and
$\dist:\mathbb{M}\times\mathbb{M}\rightarrow \mathbb{R}_{\geq 0}$ is a
distance function such that $\forall x,y,z\in\mathbb{M},$ we have i)
$\dist(x,y)=0\Leftrightarrow x=y,$ ii) $\dist(x,y)=\dist(y,x),$ and iii)
$\dist(x,z)\leq \dist(x,y)+\dist(y,z).$ If $\mathbb{M}$ is a finite set, then
we call $(\mathbb{M},\dist)$ a finite metric space.

\paraspace
\paragraph{Wasserstein-$1$ distance.}
Wasserstein-$1$ distance, also known as the Earth-Mover distance, is one of the
distance measures to quantify the similarity of two distributions, defined as
$
W(\mathcal{P}_a,\mathcal{P}_b)=\inf_{\Gamma\in\Pi(\mathcal{P}_a,\mathcal{P}_b)}\E_{(x,y)\sim \Gamma}\left(|x-y|\right),
$
where $\mathcal{P}_a$ and $\mathcal{P}_b$ are two distributions on real numbers,
and $\Pi(\mathcal{P}_a,\mathcal{P}_b)$ is the set of all joint distributions $\Gamma(x,y)$
on two real numbers whose marginal distributions are $\mathcal{P}_a$ and
$\mathcal{P}_b$, respectively.
Wasserstein-$1$ distance has been used to augment GAN's objective and
improve training stability~\cite{arjovsky2017wasserstein}. We will use it
to understand the theoretical guarantees of our method.

\paraspace
\paragraph{Logarithmic pairwise distance distribution (LPDD).}
We propose to use the pairwise distance distribution of data items to reflect
the mode structure in a dataset (\figref{pdd_tsne}-top). Since the pairwise distance is
measured under a specific metric, its distribution also depends on the metric choice.
Indeed, it has been used in~\cite{metz2016unrolled} to quantify how well
Unrolled GAN addresses mode collapse. 

Concretely, given a metric space $(\M,\dist)$, let $\mathcal{X}$ be a
distribution over $\M$, and $(\lambda,\Lambda)$ be two real values satisfying
$0<2\lambda\leq\Lambda$.
Consider two samples $x,y$ independently drawn from $\mathcal{X}$, and let
$\eta$ be the logarithmic distance between $x$ and $y$ (i.e., $\eta=\log(\dist(x,y))$).
We call the distribution of $\eta$ conditioned on
$\dist(x,y)\in[\lambda,\Lambda]$ the {\em$(\lambda,\Lambda)-$logarithmic
pairwise distance distribution} (LPDD) of the distribution
$\mathcal{X}$.
Throughout our theoretical analysis, LPDD of the distributions generated at various
steps of our method will be measured in Wasserstein-$1$ distance.

\textit{Remark.} We choose to use logarithmic distance in order to
reasonably compare two pairwise distance distributions. The rationale is illustrated in~\figref{lpdd} in the appendix.
Using logarithmic distance is also beneficial for training our GANs, which will
become clear in \secref{training}. The $(\lambda,\Lambda)$ values in the above
definition are just for the sake of theoretical rigor, irrelevant from our practical
implementation. They are meant to avoid the theoretical situation where two
samples are identical and then taking the logarithm becomes no sense.  In this
section, the reader can skip these values and refer back when reading our
theoretical analysis (in \secref{theory} and the supplementary material).


\secprespace
\subsection{Preprocessing: Subsample of Data Items}\label{sec:subsample}
\secspace
We now describe how to train BourGAN step by step.  Provided with a multiset
of data items $X=\{x_i\}_{i=1}^n$ drawn independently from an unknown
distribution $\mathcal{X}$, we first subsample $m$ ($m<n$) data items uniformly at random
from $X$. 
This subsampling step is essential,
especially when $n$ is large,
for reducing the computational cost of metric embeddings
as well as the number of dimensions of the latent space (both described in \secref{gm}).
From now on, we use $Y$ to denote the multiset of data items subsampled from $X$ (i.e., $Y\subseteq X$ and $|Y|=m$).
Elements in $Y$ will be embedded in $\ell_2$ space in the next step. 

The subsampling strategy, while simple, is theoretically sound.
Let $\mathcal{P}$ be the $(\lambda,\Lambda)$-LPDD of
the data distribution $\mathcal{X}$, and $\mathcal{P}'$ be the LPDD of the
uniform distribution on $Y$. We will show in \secref{theory} that their Wasserstein-$1$ distance
$W(\mathcal{P},\mathcal{P}')$ is tightly bounded if $m$ is sufficiently large but much smaller than $n$.
In other words, the mode structure of the real data can be captured by considering only the subsamples in $Y$.
In practice, $m$ is chosen automatically by a simple algorithm, 
which we describe in \appref{param_set}. In all our examples, we find $m=4096$ sufficient.

\secprespace
\subsection{Construction of Gaussian Mixture in Latent Space}\label{sec:gm}
\secspace
Next, we construct a Gaussian mixture model for generating random vectors in latent space.
First, we embed data items from $Y$ to an $\ell_2$ space, one that the latent random vectors reside in.
We want the latent vector dimensionality to be small,
while ensuring that the mode structure be well reflected in the latent space. 
This requires the embedding to introduce minimal distortion on the pairwise distances of data items.
For this purpose, we propose an algorithm that leverages Bourgain's embedding theorem.

\paraspace
\paragraph{Metric embeddings.}
Bourgain~\cite{b85} introduced a method that can embeds \emph{any} finite metric space into
a small $\ell_2$ space with minimal distortion. The theorem is stated as follows:
\begin{theorem}[Bourgain's theorem]\label{thm:bourgain}
Consider a finite metric space $(Y,\dist)$ with $m=|Y|.$ There exists a mapping
$g:Y \rightarrow \mathbb{R}^k$ for some $k=O(\log^2 m)$ such that
$\forall y,y'\in Y, \dist(y,y')\leq\|g(y)-g(y')\|_2\leq \alpha\cdot\dist(y,y')$,
where $\alpha$ is a constant satisfying $\alpha\leq O(\log m)$.
\end{theorem}
\vspace{-2mm}
The mapping $g$ is constructed using a randomized algorithm also given by Bourgain~\cite{b85}.
Directly applying Bourgain's theorem results in a latent space of $O(\log^2 m)$ dimensions.
We can further reduce the number of dimensions down to $O(\log m)$ through the following
corollary.
\begin{corollary}[Improved Bourgain embedding]\label{cor:bourgain}
Consider a finite metric space $(Y,\dist)$ with $m=|Y|.$ There exist a mapping
$f:Y\rightarrow \mathbb{R}^k$ for some $k=O(\log m)$ such that
$\forall y,y'\in Y, \dist(y,y')\leq\|f(y)-f(y')\|_2\leq \alpha\cdot\dist(y,y')$,
where $\alpha$ is a constant satisfying $\alpha\leq O(\log m).$
\end{corollary}
\vspace{-1.5mm}
Proved in \appref{corollary_proof},
this corollary is obtained by combining \thmref{bourgain} with the
Johnson-Lindenstrauss (JL) lemma~\cite{johnson1984extensions}.
The mapping $f$ is computed through a combination of the algorithms
for Bourgain's theorem and the JL
lemma. This algorithm of computing $f$ is detailed in \appref{bour_alg}.\todo{does this exist already?}

\vspace{-1mm}
\textit{Remark.}
Instead of using Bourgain embedding, one can find a mapping $f:Y\rightarrow
\mathbb{R}^k$ with bounded distortion, namely,
$\forall y,y'\in Y, \dist(y,y')\leq\|f(y)-f(y')\|_2\leq \alpha\cdot\dist(y,y')$,
by solving a semidefinite programming problem (e.g., see~\cite{linial1995geometry,matouvsek2002embedding}).
This approach can find an embedding with the least distortion $\alpha$.
However, solving semidefinite programming problem is much more costly than
computing Bourgain embeddings. Even if the optimal distortion factor $\alpha$ is found,
it can still be as large as $O(\log m)$ in the worst case~\cite{leighton1988approximate}.
Indeed, Bourgain embedding is optimal in the worst case.\todo{another remark about special case?}


Using the mapping $f$, we embed data items from $Y$ (denoted
as $\{y_i\}_{i=1}^m$) into the $\ell_2$ space of $k$ dimensions ($k=O(\log m)$).
Let $F$ be the multiset of the resulting vectors in $\mathbb{R}^k$ (i.e., $F=\{f(y_i)\}_{i=1}^m$).
As we will formally state in \secref{theory},
the Wasserstein-$1$ distance
between the $(\lambda,\Lambda)-$LPDD of the real data distribution $\mathcal{X}$ and
the LPDD of the uniform distribution on $F$ is tightly bounded.
Simply speaking, the mode structure in the real data is well captured by $F$ in $\ell_2$ space.


\paraspace
\paragraph{Latent-space Gaussian mixture.}
Now, we construct a distribution using $F$ to draw random vectors in latent space.
A simple choice is the uniform distribution over $F$, but such a distribution is not continuous over the latent space.
Instead, we construct a mixture of Gaussians, each of which is centered at a vector $f(y_i)$ in $F$.
In particular, we generate a latent vector $z\in\mathbb{R}^k$ in two steps:
We first sample a vector $\mu\in F$ uniformly at random, and then
draw a vector $z$ from the Gaussian distribution $\mathcal{N}(\mu,\sigma^2)$,
where $\sigma$ is a smoothing parameter that controls the smoothness of the
distribution of the latent space.
In practice, we choose $\sigma$ empirically ($\sigma=0.1$ for all our examples).
We discuss our choice of $\sigma$ in \appref{param_set}.

\emph{Remark. } By this definition, the Gaussian mixture consists of $m$ Gaussians (recall $F=\{f(y_i)\}_{i=1}^m$).
But this does not mean that we construct $m$ ``modes'' in the latent space.
If two Gaussians are close to each other in the latent space, they should be viewed as if they are from 
the same mode. It is the overall distribution of the $m$ Gaussians that reflects the distribution of modes.
In this sense, the number of modes in the latent space is implicitly defined, and the $m$ Gaussians are meant to enable
us to sample the modes in the latent space.


\secprespace
\subsection{Training}\label{sec:training}
\secspace
The Gaussian mixture distribution $\mathcal{Z}$ in the latent space
guarantees that the LPDD of $\mathcal{Z}$ is close to $(\lambda,\Lambda)-$LPDD
of the target distribution $\mathcal{X}$ (shown in \secref{theory}).
To exploit this property for avoiding mode collapse, we encourage
the generator network to match the pairwise distances of generated samples with the
pairwise $\ell_2$ distances of latent vectors in $\mathcal{Z}$.
This is realized by a simple augmentation of the GAN's objective function, namely,
\begin{align}
    & L(G, D)  =  L_{\txt{gan}} + \beta L_{\txt{dist}},\; \label{eq:ourobj} \\
    & \textrm{where} \; L_{\txt{dist}}(G)  =  \mathbb{E}_{z_i, z_j \sim \mathcal{Z}}\left[\left( \log({d(G(z_i), G(z_j))})-\log({\|z_i-z_j\|_2}) \right)^2\right] \label{eq:soft_constraint},
\end{align}
$L_{\txt{gan}}$ is the objective of the standard GAN in Eq.~\eqref{eq:ganloss},
and $\beta$ is a parameter to balance the two terms.
In $L_{\txt{dist}}$, 
$z_i$ and $z_j$ are two i.i.d. samples from $\mathcal{Z}$ conditioned on $z_i\ne z_j$.
Here the advantages of using logarithmic distances are threefold:
i) When there exists ``outlier'' modes that are far away from others, logarithmic distance prevents
those modes from being overweighted in the objective.
ii) Logarithm turns a uniform scale of the distance metric into a constant addend
that has no effect to the optimization. This is desired as
the structure of modes is invariant
under a uniform scale of distance metric.
iii) Logarithmic distances ease our theoretical analysis, which, as we will formalize in
\secref{theory}, states that when Eq.~\eqref{eq:soft_constraint} is optimized, the distribution
of generated samples will closely resemble the real distribution
$\mathcal{X}$. That is, mode collapse will be avoided.

In practice, when experimenting with real datasets, we find that
a simple pre-training step using the correspondence between $\{y_i\}_{i=1}^m$ and $\{f(y_i)\}_{i=1}^m$
helps to improve the training stability. Although not a focus of this paper, this step is described in
\appref{pretrain}.

\secprespace
\section{Theoretical Analysis}\label{sec:theory}
\secspace
This section offers an theoretical analysis of our method presented in \secref{bgn}.
We will state the main theorems here while referring to the supplementary material for their rigorous proofs.
Throughout, we assume a property of the data distribution $\mathcal{X}$:
if two samples, $a$ and $b$, are drawn independently from $\mathcal{X}$, then with a high probability ($> \nicefrac{1}{2}$)
they are distinct (i.e., $\Pr_{a,b\sim \mathcal{X}}(a\not=b)\geq \nicefrac{1}{2}$).

%

\paraspace
\paragraph{Range of pairwise distances.}
We first formalize our definition of $(\lambda,\Lambda)-$LPDD in~\secref{metric}.
Recall that the multiset $X=\{x_i\}_{i=1}^n$ is our input dataset regarded as i.i.d. samples from $\mathcal{X}$.
We would like to find a range $[\lambda,\Lambda]$ such that the pairwise distances of samples from $\mathcal{X}$
is in this range with a high probability (see Example-\ref{ex:a} and -\ref{ex:b} in \appref{theory_ex}).
Then, when considering the LPDD of $\mathcal{X}$, we account only for the pairwise
distances in the range $[\lambda, \Lambda]$ so that the logarithmic pairwise
distance is well defined.
The values $\lambda$ and $\Lambda$ are chosen by the following theorem, which we prove in \appref{thm3_proof}. 
\begin{theorem}\label{thm:range_of_lambda}
Let $\lambda=\min_{i\in[n-1]:x_{i}\not=x_{i+1}}\dist(x_{i},x_{i+1})$ and
$\Lambda=\max_{i\in[n-1]}\dist(x_{i},x_{i+1})$.
$\forall \delta,\gamma\in(0,1)$, if $n\geq C/(\delta\gamma)$ for some sufficiently large constant $C>0$, then with probability at least $1-\delta,$ $\Pr_{a,b\sim \mathcal{X}}(\dist(a,b)\in[\lambda,\Lambda]\mid \lambda,\Lambda)\geq \Pr_{a,b\sim\mathcal{X}}(a\not=b)-\gamma.$
\end{theorem}
\vspace{-1mm}
Simply speaking, this theorem states that if we choose $\lambda$ and $\Lambda$ as described above,
then we have $\Pr_{a,b\sim \mathcal{X}}(\dist(a,b)\in[\lambda,\Lambda]\mid
a\not= b)\geq 1-O(\nicefrac{1}{n})$, meaning that if $n$ is large, the pairwise distance of any
two i.i.d. samples from $\mathcal{X}$ is almost certainly in the range
$[\lambda,\Lambda]$. Therefore, $(\lambda,\Lambda)-$LPDD is a reasonable measure
of the pairwise distance distribution of $\mathcal{X}$.
In this paper, we always use $\mathcal{P}$ to denote the $(\lambda,\Lambda)-$LPDD of the
real data distribution $\mathcal{X}$.

\paraspace
\paragraph{Number of subsamples.}
With the choices of $\lambda$ and $\Lambda$, we have the following theorem to guarantee the soundness
of our subsampling step described in~\secref{subsample}.
\begin{theorem}\label{thm:main_small_sample_enough}
Let $Y=\{y_i\}_{i=1}^m$ be a multiset of $m=\log^{O(1)}(\Lambda/\lambda)\cdot \log(1/\delta)$ i.i.d. samples drawn from $\mathcal{X}$,
and let $\mathcal{P}'$ be the LPDD of the uniform distribution on $Y$.
For any $\delta\in(0,1),$
with probability at least $1-\delta,$ we have $W(\mathcal{P},\mathcal{P}')\leq O(1).$
\end{theorem}
\vspace{-2mm}
Proved in \appref{thm4_proof}, this theorem states that we only need $m$ (on the order
of $\log^{O(1)}(\Lambda/\lambda)$)
subsamples to form a multiset $Y$ that well captures the mode structure in the real data\todo{how to argue $m<n$?}.

\begin{wrapfigure}[10]{r}{0.35\textwidth}
\vspace{-5mm}
\centering
\includegraphics[width=0.35\textwidth]{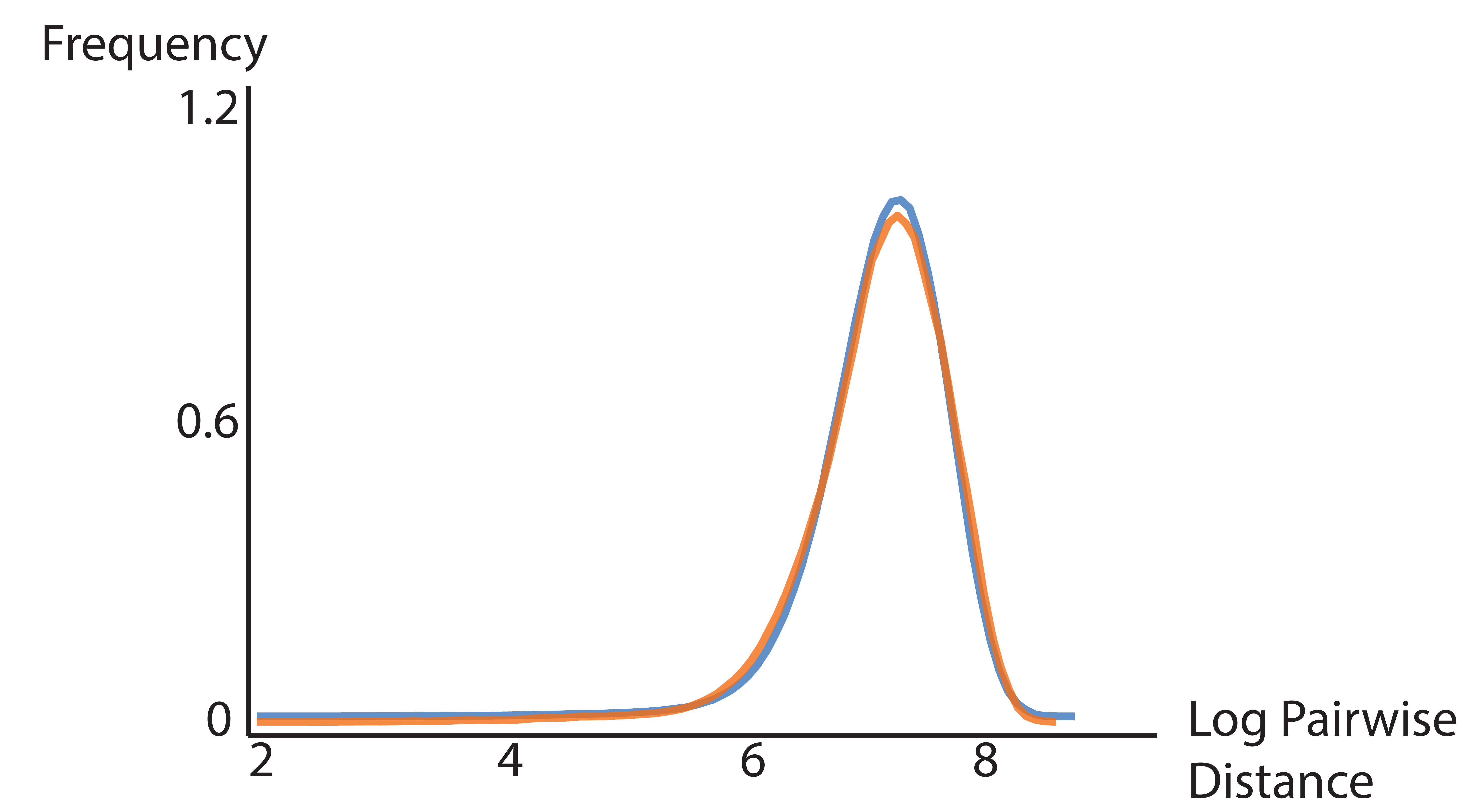}
\vspace{-6mm}
\caption{LPDD of uniform distribution $\mathcal{F}$ (orange) and of samples from Gaussian mixture (blue). 
}\label{fig:dist_distr}
\end{wrapfigure}
\paraspace
\paragraph{Discrete latent space.}
Next, we lay a theoretical foundation for our metric embedding step described in
\secref{gm}.  Recall that $F$ is the multiset of vectors resulted from
embedding data items from $Y$ to the $\ell_2$ space (i.e., $F=\{f(y_i)\}_{i=1}^m)$.
As proved in \appref{thm5_proof}, we have:
\begin{theorem}\label{thm:main_distance_of_distance_distribution}
Let $\mathcal{F}$ be the uniform distribution on the multiset $F$.
Then with probability at least $0.99$,
we have $W(\mathcal{P},\hat{\mathcal{P}})\leq
O(\log\log\log (\Lambda/\lambda)),$
where $\hat{\mathcal{P}}$ is the LPDD of $\mathcal{F}$.
 \end{theorem}
\vspace{-2mm}
Here the triple-log function ($\log\log\log (\Lambda/\lambda)$)
indicates that the Wasserstein distance bound can be very tight. Although this theorem
states about the uniform distribution on $F$, not precisely the Gaussian mixture we constructed,
it is about the case when $\sigma$ of the Gaussian mixture approaches zero. We also
empirically verified the consistency of LPDD from Gaussian mixture samples (\figref{dist_distr})\todo{label on y-axis 1.2?}.

\paraspace
\paragraph{GAN objective.}
Next, we theoretically justify the objective function (i.e., Eq.~\eqref{eq:soft_constraint} in \secref{training}).
Let $\mathcal{\tilde{X}}$ be the distribution of generated samples $G(z)$ for $z\sim\mathcal{Z}$
and $\mathcal{\tilde{P}}$ be the $(\lambda,\Lambda)-$LPDD of $\mathcal{\tilde{X}}$.
Goodfellow et al.~\cite{goodfellow2014generative} showed that the global optimum of the
GAN objective~\eqref{eq:ganloss} is reached if and only if $\mathcal{\tilde{X}}=\mathcal{X}$.
Then, when this optimum is achieved, we must also have $W(\mathcal{P},\mathcal{\tilde{P}})=0$
and $W(\mathcal{\tilde{P}},\mathcal{\hat{P}})\leq O(\log\log\log(\Lambda/\lambda))$.
The latter is because
\mbox{
$W(\mathcal{P},\mathcal{\hat{P}})\leq O(\log\log\log(\Lambda/\lambda))$
from \thmref{main_distance_of_distance_distribution}.}

As a result, the GAN's minmax problem~\eqref{eq:ganloss} is equivalent to the constrained minmax problem,
$\min_G\max_D L_{\txt{gan}}(G,D)$,
subject to $W(\mathcal{\tilde{P}},\mathcal{\hat{P}})\leq \beta$, 
where $\beta$ is on the order of $O(\log\log\log(\Lambda/\lambda))$.
Apparently, this constraint renders the minmax problem harder. We therefore
consider the minmax problem, $\min_G\max_D L_{\txt{gan}}(G,D)$,
subjected to slightly strengthened constraints,
\begin{align}
    & \forall z_1\not=z_2\in \supp(\mathcal{Z}), \dist(G(z_1),G(z_2))\in[\lambda,\Lambda],\,\,\textrm{and}\label{eq:range_constraint}\\
    & \left[\log(\dist(G(z_1),G(z_2)))-\log \|z_1-z_2\|_2\right]^2\leq \beta^2\label{eq:distance_constraint}.
\end{align}
As proved in \appref{strengthen}, if the above constraints are
satisfied, then $W(\mathcal{\tilde{P}},\mathcal{\hat{P}})\leq \beta$ is automatically satisfied.
In our training process, we assume that the
constraint~\eqref{eq:range_constraint} is automatically satisfied, supported by \thmref{range_of_lambda}.
Lastly, instead of using Eq.~\eqref{eq:distance_constraint} as a hard constraint, we treat it as a
soft constraint showing up in the objective function~\eqref{eq:soft_constraint}. From this perspective,
the second term in our proposed objective~\eqref{eq:ourobj} can be interpreted as a Lagrange multiplier of the constraint.

\paragraph{LPDD of the generated samples.}
Now, if the generator network is trained to satisfy the
constraint~\eqref{eq:distance_constraint}, we have
 $W(\mathcal{\tilde{P}},\mathcal{\hat{P}})\leq O(\log\log\log(\Lambda/\lambda))$.
Note that this satisfaction does \emph{not} imply that the global optimum of the
GAN in Eq.~\eqref{eq:ganloss} has to be reached
-- such a global optimum is hard to achieve in practice.
Finally, using the triangle inequality of the Wasserstein-$1$ distance and \thmref{main_distance_of_distance_distribution},
we reach the conclusion that
\begin{equation}\label{eq:final_bound}
W(\tilde{\mathcal{P}},\mathcal{P})\leq
W(\tilde{\mathcal{P}},\hat{\mathcal{P}})+W(\mathcal{P},\hat{\mathcal{P}})\leq
O(\log\log\log(\Lambda/\lambda)).
\end{equation}
This means that the LPDD of generated samples closely resembles that of the data distribution.
To put the bound in a concrete context,
in Example~\ref{ex:thought} of \appref{theory_ex}, we analyze a toy case 
in a thought experiment to show,
if the mode collapse occurs (even partially), how large $W(\tilde{\mathcal{P}},\mathcal{P})$
would be in comparison to our theoretical bound here.

\secprespace
\section{Experiments}
\secspace
This section presents the empirical evaluations of our method. 
There has not been a consensus on how to evaluate GANs
in the machine learning community~\cite{theis2015note, borji2018pros},
and quantitative measure of mode collapse is also not straightforward.
We therefore evaluate our method using both synthetic and real datasets, most of which have
been used by recent GAN variants. 
We refer the reader to \appref{eval_app} for detailed
experiment setups and complete results, while highlighting our main findings here.

\begin{table}[b]
\centering
\scalebox{0.85}{
\begin{tabular}{@{}llllllllllllllll@{}}
    \whline{1.0pt}
                 & \multicolumn{3}{c}{2D Ring}                                                                                                                                                        & \multicolumn{3}{c}{2D Grid}                                                                                                     & \multicolumn{3}{c}{2D Circle}                                   \\ 
                 & \multicolumn{1}{c}{\begin{tabular}[c]{@{}c@{}}\#modes\\ (max 8)\end{tabular}} & $\mathcal{W}_1$   & \multicolumn{1}{c}{\begin{tabular}[c]{@{}c@{}}low \\ quality\end{tabular}} & \begin{tabular}[c]{@{}l@{}}\#modes\\ (max 25)\end{tabular} & $\mathcal{W}_1$   & \multicolumn{1}{c}{\begin{tabular}[c]{@{}c@{}}low \\ quality\end{tabular}} & \begin{tabular}[c]{@{}l@{}}center \\ captured\end{tabular} & $\mathcal{W}_1$   & \multicolumn{1}{c}{\begin{tabular}[c]{@{}c@{}}low \\ quality\end{tabular}}  \\ 
    \whline{0.7pt}
        GAN      & 1.0                                                                           & 38.60 & 0.06\%                                                                     & 17.7                                                       & 1.617  														& 17.70\%                                                                          & No                                                                            & 32.59 & 0.14\% \\
        Unrolled & 7.6                                                                           & 4.678 & 12.03\%                                                                    & 14.9                                                       & 2.231 														& 95.11\%                                                                          & No                                                                            & 0.360 & 0.50\% \\
        VEEGAN   & 8.0                                                                           & 4.904 & 13.23\%                                                                    & 24.4                                                       & 0.836  														& 22.84\%                                                                     	   & Yes                                                                           & 0.466 & 10.72\% \\
        PacGAN   & 7.8                                                                           & 1.412 & 1.79\%                                                                     & 24.3                                                       & 0.898  														& 20.54\%                                                                          & Yes                                                                           & 0.263 & 1.38\% \\ 
    \whline{0.7pt}
        \textbf{BourGAN}  & 8.0                                                                  & 0.687 & 0.12\%                                                                     & 25.0                                                       & 0.248  														& 4.09\%                                                                           & Yes                                                                           & 0.081 & 0.35\% \\ 
    \whline{1.0pt}
    \end{tabular}
}
\caption{Statistics of Experiments on Synthetic Datasets}\label{tab:synthetic}
\end{table}

\begin{figure}[t]
\centering
\includegraphics[width=0.92\textwidth]{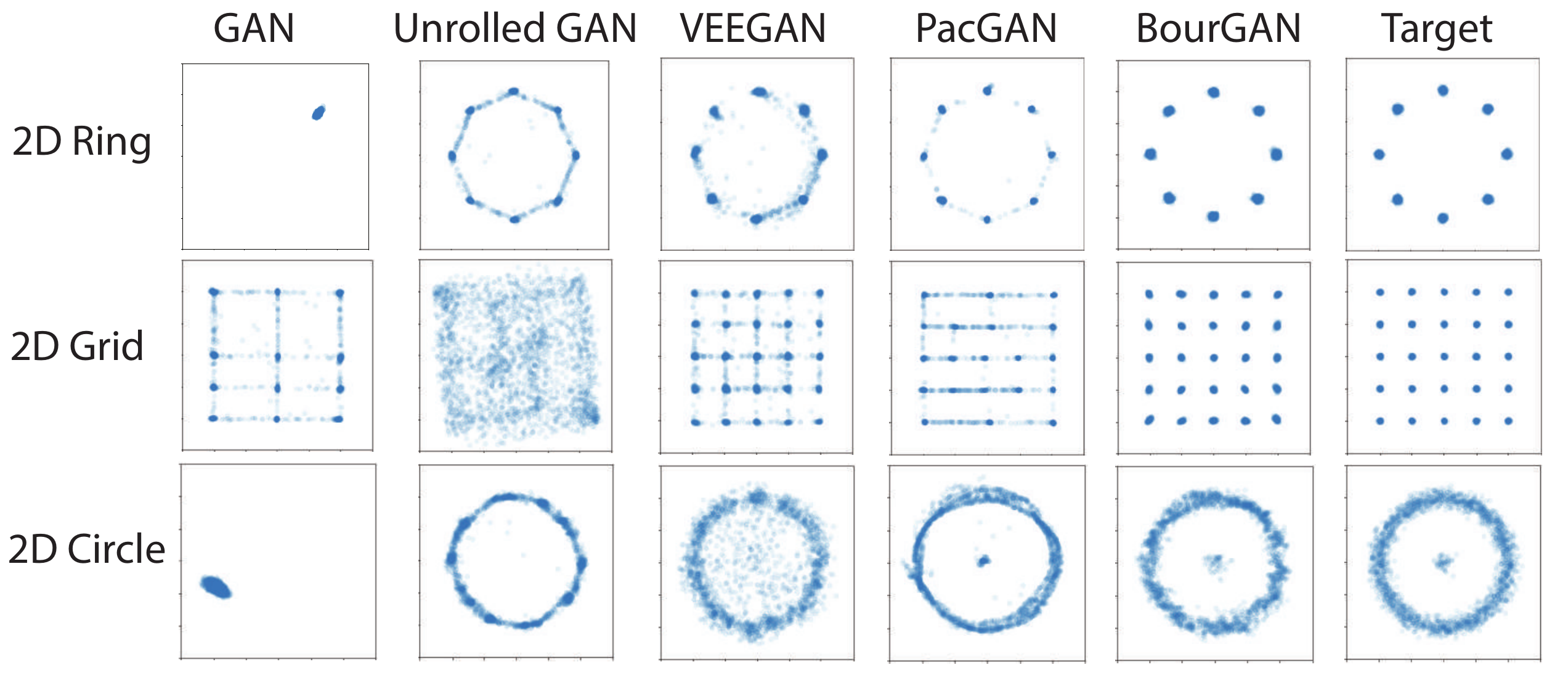}
\vspace{-1mm}
\caption{\textbf{Synthetic data tests}. In all three tests, our method clearly captures 
all the modes presented in the targets, while producing \emph{no} unwanted samples located between the regions of modes.}\label{fig:synthetic}
\vspace{-1mm}
\end{figure}

\paraspace
\paragraph{Overview.}
We start with an overview of our experiments. 
\textbf{i)} On synthetic datasets, we quantitatively compare our method 
with four types of GANs, including the original
GAN~\cite{goodfellow2014generative} and more recent
VEEGAN~\cite{srivastava2017veegan}, Unrolled GANs~\cite{metz2016unrolled}, and
PacGAN~\cite{lin2017pacgan}, following the evaluation metrics used by those methods (\appref{synthetic_details}).
\textbf{ii)} We also examine in each mode how well the distribution of generated samples 
matches the data distribution (\appref{synthetic_details}) -- a new test not presented previously.
\textbf{iii)} We compare the training convergence rate of our method with existing
GANs (\appref{synthetic_details}), examining to what extent the Gaussian mixture sampling is beneficial.
\textbf{iv)} We challenge our method with the difficult \emph{stacked MNIST} dataset (\appref{mnist_app}),
testing how many modes it can cover. 
\textbf{v)} Most notably, we examine if there are ``false positive'' samples
generated by our method and others (\figref{synthetic}). Those are unwanted
samples not located in any modes.
In all these comparisons, we find that BourGAN clearly produces higher-quality samples.
In addition, we show that \textbf{vi)} our method is able to incorporate 
different distance metrics, ones that lead to different mode interpretations (\appref{mnist_app});
and \textbf{vii)} our pre-training step (described in \appref{pretrain}) further accelerates
the training convergence in real datasets (\appref{synthetic_details}).
Lastly, \textbf{viii)} we present some qualitative results (\appref{qual}).
\vspace{2mm}



\paraspace
\paraspace
\paragraph{Quantitative evaluation.}
We compare BourGAN with other methods on three synthetic datasets:
eight 2D Gaussian distributions arranged in a ring (2D Ring), 
twenty-five 2D Gaussian distributions arranged in a grid (2D Grid), 
and a circle surrounding a Gaussian placed in the center (2D Circle).
The first two were used in previous methods~\cite{metz2016unrolled, srivastava2017veegan,lin2017pacgan}, and the last is proposed by us.
The quantitative performance of these methods are summarized in \tabref{synthetic}, where the column ``\# of modes''
indicates the average number of modes captured by these methods, and ``low quality'' indicates number of samples
that are more than $3\times$ standard deviations away from the mode centers. 
Both metrics are used in previous methods~\cite{srivastava2017veegan,lin2017pacgan}. For the 2D circle case,
we also check if the central mode is captured by the methods. Notice that all
these metrics measure how many modes are captured, but \emph{not} how well the
data distribution is captured.  To understand this, we also compute the
Wasserstein-$1$ distances between the distribution of generated samples and the
data distribution (reported in \tabref{synthetic}). It is evident that our
method performs the best on all these metrics (see \appref{synthetic_details} for more details).

\paraspace
\paragraph{Avoiding unwanted samples.}
A notable advantage offered by our method is the ability to avoid \emph{unwanted} samples, ones that
are located between the modes. We find that all the four existing GANs suffer from this problem (see \figref{synthetic}),
because they use Gaussian to draw latent vectors (recall \figref{gaussian_flaw}). In contrast, our method generates 
no unwanted samples in all three test cases.
We refer the reader to \appref{mnist_app} for a detailed discussion of this feature and several other quantitative comparisons.


\paraspace
\paragraph{Qualitative results.}
We further test our algorithm on real image datasets. \figref{celebA} illustrates a qualitative comparison between 
DCGAN and our method, both using the same generator and discriminator architectures and default hyperparameters.
\appref{qual} includes more experiments and details.


\begin{figure}[t]
\centering
\includegraphics[width=0.92\textwidth]{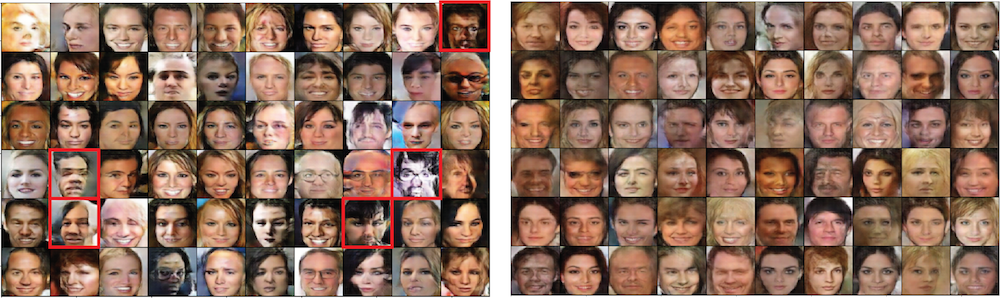}
\vspace{-1mm}
\caption{\textbf{Qualitative results} on CelebA dataset using DCGAN (Left) and BourGAN (Right) under $\ell_2$ metric. 
It appears that DCGAN generates some samples that are visually more implausible (e.g., red boxes)
in comparison to BourGAN. Results are fairly sampled at random, not cherry-picked.}\label{fig:celebA}
\vspace{-1mm}
\end{figure}



\secprespace
\section{Conclusion}
\secspace
This paper introduces BourGAN, a new GAN variant aiming to address mode
collapse in generator networks.
In contrast to previous approaches, we draw latent space vectors using a 
Gaussian mixture, which is constructed through metric embeddings. Supported 
by theoretical analysis and experiments, our method enables
a well-posed mapping between latent space and multi-modal data distributions. 
In future, our embedding and Gaussian mixture sampling can also be 
readily combined with other GAN variants and even other generative models to leverage their advantages.

\subsection*{Acknowledgements}
We thank Daniel Hsu, Carl Vondrick and Henrique Maia for the helpful feedback.
Chang Xiao and Changxi Zheng are supported in part by the National Science Foundation
(CAREER-1453101, 1717178 and 1816041) and generous donations from SoftBank and
Adobe.  Peilin Zhong is supported in part by National Science Foundation (CCF-1703925,
CCF-1421161, CCF-1714818, CCF-1617955 and CCF-1740833), Simons Foundation
(\#491119 to Alexandr Andoni) and Google Research Award.


\bibliographystyle{unsrt}
\bibliography{ref}

\newpage
\appendix
\begin{algorithm}[t]\caption{Improved Bourgain Embedding}\label{alg:bourgain}
 \small
\begin{algorithmic}
\small
\STATE \textbf{Input:} A finite metric space $(Y,\dist).$
\STATE \textbf{Output:} A mapping $f:Y\rightarrow \mathbb{R}^{O(\log |Y| )}$.
\STATE //\textbf{Bourgain Embedding:}
\STATE Initialization: $m\leftarrow |Y|,$ $t\leftarrow O(\log m),$ and $\forall i\in[\lceil\log m\rceil],j\in [t],S_{i,j}\leftarrow\emptyset.$
\FOR{$i=1\rightarrow \lceil\log m\rceil$}
\FOR{$j=1\rightarrow t$}
        \STATE For each $x\in Y,$ independently choose $x$ in $S_{i,j},$ i.e. $S_{i,j}=S_{i,j}\cup \{x\}$ with probability $2^{-i}.$
    \ENDFOR
\ENDFOR
\STATE Initialize $g:Y\rightarrow \mathbb{R}^{\lceil\log m\rceil\cdot t}.$
\FOR{$x\in Y$}
    \STATE $\forall i\in[\lceil\log m\rceil],j\in[t],$ set the $((i-1)\cdot t+j)$-th coordinate of $g(x)$ as $\dist(x,S_{i,j}).$
\ENDFOR
\STATE //\textbf{Johnson-Lindenstrauss Dimentionality Reduction:}
\STATE Let $d=O(\log m)$, and let $G\in\mathbb{R}^{d\times (\lceil\log m\rceil\cdot t)}$ be a random matrix with entries drawn from i.i.d. $\mathcal{N}(0,1).$
\STATE Let $h:\mathbb{R}^{\lceil\log m\rceil\cdot t}\rightarrow \mathbb{R}^d$ satisfy $\forall x\in\mathbb{R}^{\lceil\log m\rceil\cdot t}, h(x)\leftarrow G\cdot x.$
\STATE //\textbf{Rescaling:}
\STATE Let $\beta=\min_{x,y\in Y:x\not=y}\frac{\|h(g(x))-h(g(y))\|_2}{\dist(x,y)}.$
\STATE Initialize $f:Y\rightarrow \mathbb{R}^{d}.$ For $x\in Y,$ set $f(x)\leftarrow h(g(x))/\beta.$
\STATE Return $f$.
\end{algorithmic}
\end{algorithm}

\begin{figure}[b]
\centering
\begin{overpic}[width=1.0\textwidth]{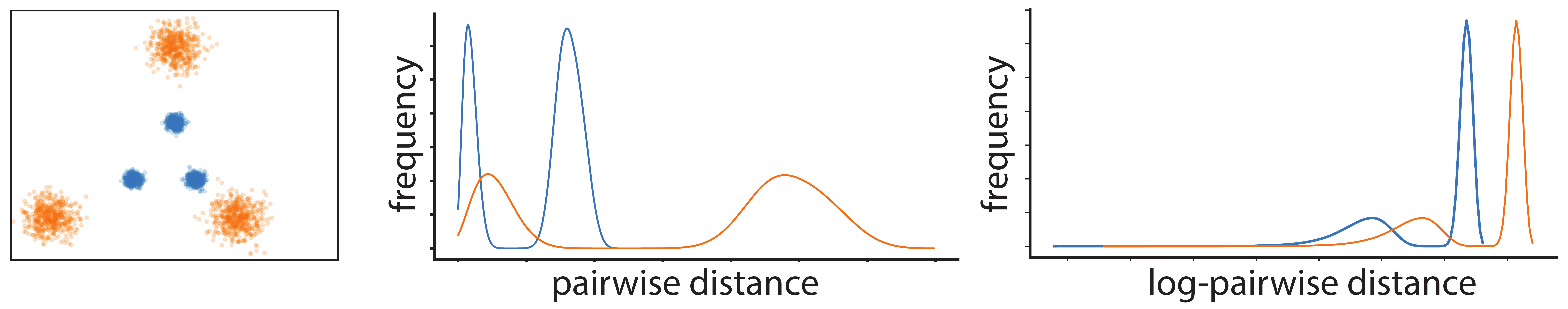}
    \put(10.3,-1.7){(a)}
    \put(42.9,-1.7){(b)}
    \put(81.0,-1.7){(c)}
\end{overpic}
\vspace{-2mm}
\caption{ \textbf{Intuition of using LPDD.}
\textbf{(a)} Here blue points illustrate a dataset with three modes.
The orange points indicate the same data but uniformly scaled up.
\textbf{(b)} The pairwise distance distributions of both datasets are different.
The distribution of orange points is a streched version of the distribution of blue points.
As a result, the Wasserstein-$1$ distance between both distributions can become arbitrarily large,
depending on the scale.
\textbf{(c)} In contrast, the distribution of logarithmic pairwise distance remains the same up to a
constant shift. In this case, the Wasserstein-$1$ distance of the logarithmic pairwise distance distributions
is differed by only a constant addent, which can be easily accounted.}\label{fig:lpdd}
\end{figure}

\section{Algorithm of Improved Bourgain Embedding}\label{sec:bour_alg}
Algorithm~\ref{alg:bourgain} outlines our randomized algorithm that computes
the improved Bourgain embedding with high probability.
To embed a finite metric space $(Y,\dist)$ into $\ell_2$ space,
Algorithm~\ref{alg:bourgain} takes $O(m^2\cdot s + m^2\log^2 m)$ running time,
where $m=|Y|$ is the size of $Y$, and $s$ is the running time needed to compute a
pairwise distance $\dist(x,y)$ for any $x,y\in Y.$

\section{Proof of Corollary~\ref{cor:bourgain}}\label{sec:corollary_proof}
Here we prove the Corollary~\ref{cor:bourgain} introduced in \secref{gm}.
First, we recall the Johnson-Lindenstrauss lemma~\cite{johnson1984extensions}.
\begin{theorem}[Johnson-Lindenstrauss lemma]\label{thm:JL_lemma}
    Consider a set of $m$ points $X=\{x_i\}_{i=1}^m$ in a vector space $\mathbb{R}^t$. There
exist a mapping $h:X \rightarrow \mathbb{R}^k$ for some $k=O(\log m)$ such that
\begin{align*}
\forall i,j\in[m], \|h(x_i)-h(x_j)\|_2\leq \|x_i-x_j\|_2\leq O(1)\cdot\|h(x_i)-h(x_j)\|_2.
\end{align*}
\end{theorem}
By combining this lemma with Bourgain's theorem~\ref{thm:bourgain},
we reach the corollary through the following proof.
\vspace{-3mm}
\begin{proof}
By Theorem~\ref{thm:bourgain}, we can embed all data items from $Y$ into
the $\ell_2$ space with $O(\log^2 m)$ dimensions and with $O(\log m)$ distortion. Then,
according to Theorem~\ref{thm:JL_lemma}, we can further reduce the number of dimensions to $O(\log m)$
with $O(\log m)$ distortion.
\end{proof}

\section{Pre-training}\label{sec:pretrain}
While our method addresses mode collapse, in practice, we have to confront other challenges
of training the GAN, particularly its instability and sensitivity to hyper-parameters.
To this end, we pre-train the generator network $G$ and use it to warm start
the training of our GAN. Pre-training is made possible because our metric embedding step has
established the correspondence between the embedding vectors $f(y_i)$
in the latent space and the data items $y_i\in Y$, $i\in[m]$. This correspondence allows us to perform
a supervised learning to minimize the objective
\begin{align*}
L_{\txt{pre}}(G) = \mathbb{E}_{y_i, z_i} \left[d(G(f(y_i)), y_i)\right].
\end{align*}
As will be shown in our experiments, this pre-training step leads to faster convergence when
we train our GANs.
Lastly, we note that our method can be straightforwardly combined with other
objective function extensions~\cite{mao2017least,arjovsky2017wasserstein,gulrajani2017improved,zhao2016energy, saatci2017bayesian, li2017mmd} and network architectures~\cite{lin2017pacgan, tolstikhin2017adagan, metz2016unrolled}, 
ones that specifically focus on addressing other challenges such as instability, to leverage their advantages.

\section{Illustrative Examples for \secref{theory}}\label{sec:theory_ex}
The following two examples illustrate the ranges of the pairwise distance that can cover a pairwise distance
sample with a high probability. They are meant to exemplify the choices of $\lambda$ and $\Lambda$ discussed
in \secref{theory}.
\begin{example}\label{ex:a}
 Consider the set of all points in $\mathbb{R}^{20},$ and the distance measure is chosen to be the Euclidean distance.
 Let $\mathcal{X}$ be the Gaussian distribution $\mathcal{N}(0,I).$
 Suppose we draw two i.i.d. samples $x,y$ form $\mathcal{X},$
 then with probability at least $0.99999,$ $\dist(x,y)$ should be in the range $[0.1,10].$
 \end{example}

 \begin{example}\label{ex:b}
 Consider the set of all $256 \times 256$ grayscale images, and the brightness of each pixel is described by a number in $\{0,1,2,\cdots,255\}.$
 Let $\mathcal{X}$ be a uniform distribution over all the images which contains a cat.
 Suppose we draw two i.i.d. samples $x,y$ from $\mathcal{X},$ then with probability $1,$ the distance between $x$ and $y$ should be in the range $[1,255\cdot 256\cdot 256]=[1,16777216].$
 \end{example}

Next, we show a concrete example in which if the generator produces samples mainly in one
mode, then $W(\mathcal{P},\mathcal{\tilde{P}})$ can be as large as $\Omega(\log(\Lambda/\lambda))$,
drastically larger than the bound in \eqref{eq:final_bound}.
\begin{example}\label{ex:thought}
Suppose $\M=A\cup B\subset \mathbb{R}^d,$ where $A=\{0,1\}^d$ is a Hamming cube
close to the origin, and $B=\{\nicefrac{\Lambda}{\sqrt{d}}-1,\nicefrac{\Lambda}{\sqrt{d}}\}^d$
is another Hamming cube far away from the origin (i.e., $\Lambda\gg d$).
It is easy to see that $A,B$ are two separated modes. Let
$\dist:\M\times\M\rightarrow \mathbb{R}_{\geq 0}$ be the Euclidean distance
(i.e., $\forall x,y\in\M$, $\dist(x,y)=\|x-y\|_2$), and let $\lambda=1.$
It is easy to see that $\forall x\not=y\in\M,$ we have
$\dist(x,y)\in[\lambda,\Lambda].$ Suppose the real data distribution
$\mathcal{X}$ is the uniform distribution on $\M$. Also suppose the distribution of
generated samples is $\tilde{\mathcal{X}}$, and the probability that generator $G$
generates samples near the mode $B$ is at most $1/10$. Then, consider the
$(\lambda,\Lambda)-$LPDD (denoted by $\mathcal{P}$) of
$\mathcal{X}$. If we draw two independent samples from $\mathcal{X}$,
then conditioned on this two samples being distinct, with probability at least
$\nicefrac{1}{3}$, they are in different modes. Thus, if we draw a sample $p$
from $\mathcal{P},$ then with probability at least $\nicefrac{1}{3}$, $p$ is at least
$\nicefrac{\Lambda}{2}$. Now consider the distribution $\tilde{\mathcal{X}}$ of generated samples. Since with probability at
least $\nicefrac{9}{10}$, a sample from $\tilde{\mathcal{X}}$ will land in mode $A$, if we
draw two samples from $\tilde{\mathcal{X}},$ then with probability at least
$\nicefrac{4}{5}$, the distance between these two samples is at most $\sqrt{d}$. Thus, the
Wasserstein distance is at least
$(\nicefrac{4}{5}-(1-\nicefrac{1}{3}))\cdot|\log(\frac{\Lambda}{2})-\log\sqrt{d}|\geq
0.1\log(\nicefrac{\Lambda}{\sqrt{d}})=\Omega(\log(\Lambda/\lambda))$.
\end{example}

\section{Strengthened Constraints for GAN's Minmax Problem}\label{sec:strengthen}
As explained in \secref{theory}, introducing the constraint $W(\mathcal{P},\mathcal{P}')<\beta$
in the GAN optimization makes the problem harder to solve.
Thus, we choose to slightly strengthen the constraint.
Observe that if for all $z_1\not=z_2\in \supp(\mathcal{Z})$ we have
$|\log(\dist(G(z_1),G(z_2)))-\log(\|z_1-z_2\|_2)|\leq
O(\log\log\log(\Lambda/\lambda))$ and
$\dist(G(z_1),G(z_2))\in[\lambda,\Lambda],$ we have
\begin{align*}
 W(\mathcal{\tilde{P}},\mathcal{\hat{P}})&\leq \sum_{z_1\not=z_2\in\supp(\mathcal{Z})} \Pr_{Z_1,Z_2\sim \mathcal{Z}}(Z_1=z_1,Z_2=z_2\mid Z_1\not =Z_2)\cdot \left|\log\left(\frac{\dist(G(z_1),G(z_2))}{\|z_1-z_2\|_2}\right)\right|\\
 &\leq O(\log\log\log(\Lambda/\lambda)).
\end{align*}
In other words, if the constraints in~\eqref{eq:range_constraint} and~\eqref{eq:distance_constraint}
are satisfied, then the constraint $W(\mathcal{P},\mathcal{P}')<\beta$ is automatically satisfied.
Thus, they are a slightly strengthened version of $W(\mathcal{P},\mathcal{P}')<\beta$.

\section{Evaluation and Experiment} \label{sec:eval_app}
In this section, we provide details of our experiments,
starting with a few implementation details that are worth noting.
All our experiments are performed using a Nvidia GTX 1080 Ti Graphics card and
implemented in Pytorch~\cite{paszke2017automatic}.

\subsection{Parameter setup}\label{sec:param_set}
As discussed in \secref{subsample}, we randomly sample $m$ data items from the provided the dataset
to form the set $Y$ for subsequent metric embeddings. In our implementation, we choose $m$ automatically by using a simple
iterative algorithm. Starting from a small $m$ value (e.g., 32), in each iteration we double $m$
and add more samples from the real dataset. We stop the iteration when the pairwise distance distribution
of the samples converges under the Wasserstein-1 distance. The termination of this process is guaranteed 
because of the existence of the theoretical upper bound of $m$ (recall \thmref{main_small_sample_enough}).
In all our examples, we found $m=4096$ sufficient.
With the chosen $m$, we construct the multiset $Y={y_i}_{i=1}^m$ by uniformly sampling the dataset $X$.
Afterwards, we compute the metric embedding $f(y_i)$ for each $y_i\in Y$,
and normalize each vector in $\{f(y_i)\}_{i=1}^m$ by
\[
    \bar{f}(y_i) = \frac{f(y_i)-\mu_0}{\sigma_0},
\]
where $\mu_0$ and $\sigma_0$ are the average and standard deviation of the entire set $\{f(y_i)\}_{i=1}^m$, respectively.

Two other parameters are needed in our method, namely, $\beta$ in Eq.~\eqref{eq:ourobj}
and the standard deviation $\sigma$ used for the sampling latent Gaussian
mixture model (recall \secref{gm}).  In all our experiments, we set $\beta=0.2$
and $\sigma=0.1$. We find that the final mode coverage of generated samples is not sensitive to $\sigma$ value
in the range $[0.2, 0.6]$. Only when $\sigma$ is too small, the Gaussian mixture becomes noisy (or ``spiky''), 
and when $\sigma$ is too large, the Gaussian mixture starts to degrade into a single Gaussian as used in conventional GANs.

\subsection{Experiment Details on Synthetic Data}\label{sec:synthetic_details}
\paragraph{Setup.}
We follow the experiment setup used in~\cite{srivastava2017veegan} for 2D Ring
and 2D Grid.  In the additional 2D circle case, the input dataset is generated by using
100 Gaussian distributions on a circle with a radius $r=2$, as well as three
identical Gaussians located at the center of the circle. All Gaussians have the
same standard deviation (i.e., 0.05).

All the GANs (including our method and compared methods) in this experiment share
the same generator and discriminator architectures.
They have two hidden layers, each of which has 128 units with ReLU activation
and without any dropout \cite{srivastava2014dropout} or normalization layers \cite{ioffe2015batch}.
When using the Unrolled GAN~\cite{metz2016unrolled}, we set the number of
unrolling steps to be five as suggested in the authors' reference
implementation.
When using PacGAN~\cite{lin2017pacgan}, we follow the authors' suggestion and
set the number of packing to be four.  In all synthetic experiments, our method
is performed without the pre-training step described in \secref{pretrain}.

During training, we use a mini-batch size of 256 with 3000 iterations in total,
and use the Adam~\cite{kingma2014adam} optimization method with a learning rate
of $0.001$ and set $\beta_1=0.5, \beta_2=0.999$. During testing, we use 2500
samples from the learned generator network for evaluation, and use $\ell_2$ distance
as the target distance metric for Bourgain embedding.
Every metric value listed in \tabref{synthetic} is evaluated and averaged over 10 trials.

\paragraph{Studies.}
When evaluating the number of captured modes (``\# modes'' in \tabref{synthetic}),
a mode is considered as being ``captured'' when there exists at least one sample located
within one standard-deviation-distance (1-std) away from the center of the mode.
This criterion is slightly different from that used in~\cite{srivastava2017veegan, lin2017pacgan},
in which they use three standard-deviation (3-std).
We choose to use 1-std because we would like to have finer granularity to differentiate
the tested GANs in terms of their mode capture performance.

To gain a better understanding of the mode capture performance, we also measure
in each method the percentages of generated samples located within 1-, 2-, and 3-std away
from mode centers for the three test datasets. The results are reported in \tabref{synthetic_app}.
We note that for Gaussian distribution, the percentages of samples located in 1-, 2-, and 3-std away
from the center are 68.2\%, 95.4\%, 99.7\%, respectively~\cite{pukelsheim1994three}.
Our method produces results that are closest to these percentages in comparison to other methods.
This suggests that our method better captures not only individual modes but also the data \emph{distribution}
in each mode, thanks to the pairwise distance preservation term
\eqref{eq:soft_constraint} in our objective function. We also note that this experiment result is echoed
by the Wasserstein-$1$ measure reported in~\tabref{synthetic}, for which we measure the Wasserstein-$1$ distance
between the distribution of generated samples and the true data distribution. Our method under that metric also performs
the best.



\begin{table}[t]
    \centering
    \scalebox{0.9}{
    \begin{tabular}{@{}lllllllllll@{}}
    \whline{1.0pt}
                 & \multicolumn{3}{c}{2D Ring}                   & \multicolumn{3}{c}{2D Grid}  & \multicolumn{3}{c}{2D Circle}  \\
                 &   1-std   & 2-std   & 3-std    & 1-std   & 2-std   & 3-std    & 1-std   & 2-std   & 3-std \\ 
    \whline{0.7pt}
        GAN      &   61.46\% & 96.14\% & 99.94\%  & 35.86\% & 69.86\% & 82.3\%     & 82.08\% & 98.26\% & 99.86\%       &\\
        Unrolled &   70.66\% & 85.09\% & 87.96\%  & 0.54\%  & 2.10\%  & 4.88\%    & 92.08\% & 99.35\% & 99.49\%      & \\
        VEEGAN   &   51.68\% & 79.24\% & 86.76\%  & 24.76\% & 60.24\% & 77.16\%   & 54.72\% & 80.44\% & 89.28\%       &\\
        PacGAN   &   88.32\% & 97.28\% & 98.20\%  & 28.9\%  & 67.76\% & 79.46\%    & 58.10\% & 94.62\% & 98.62\%       &\\ 
    \whline{0.7pt}
        \textbf{BourGAN} & 59.54\% & 96.64\% & 99.88\%  & 38.64\% & 81.54\% & 95.9\% & 67.52\% & 95.64\% & 99.64\%       & \\ 
    \whline{1.0pt}
    \end{tabular}
}
\caption{Statistics of Experiments on Synthetic Datasets}\label{tab:synthetic_app}
\vspace{-4mm}
\end{table}

\begin{figure}
\centering
\includegraphics[width=0.99\textwidth]{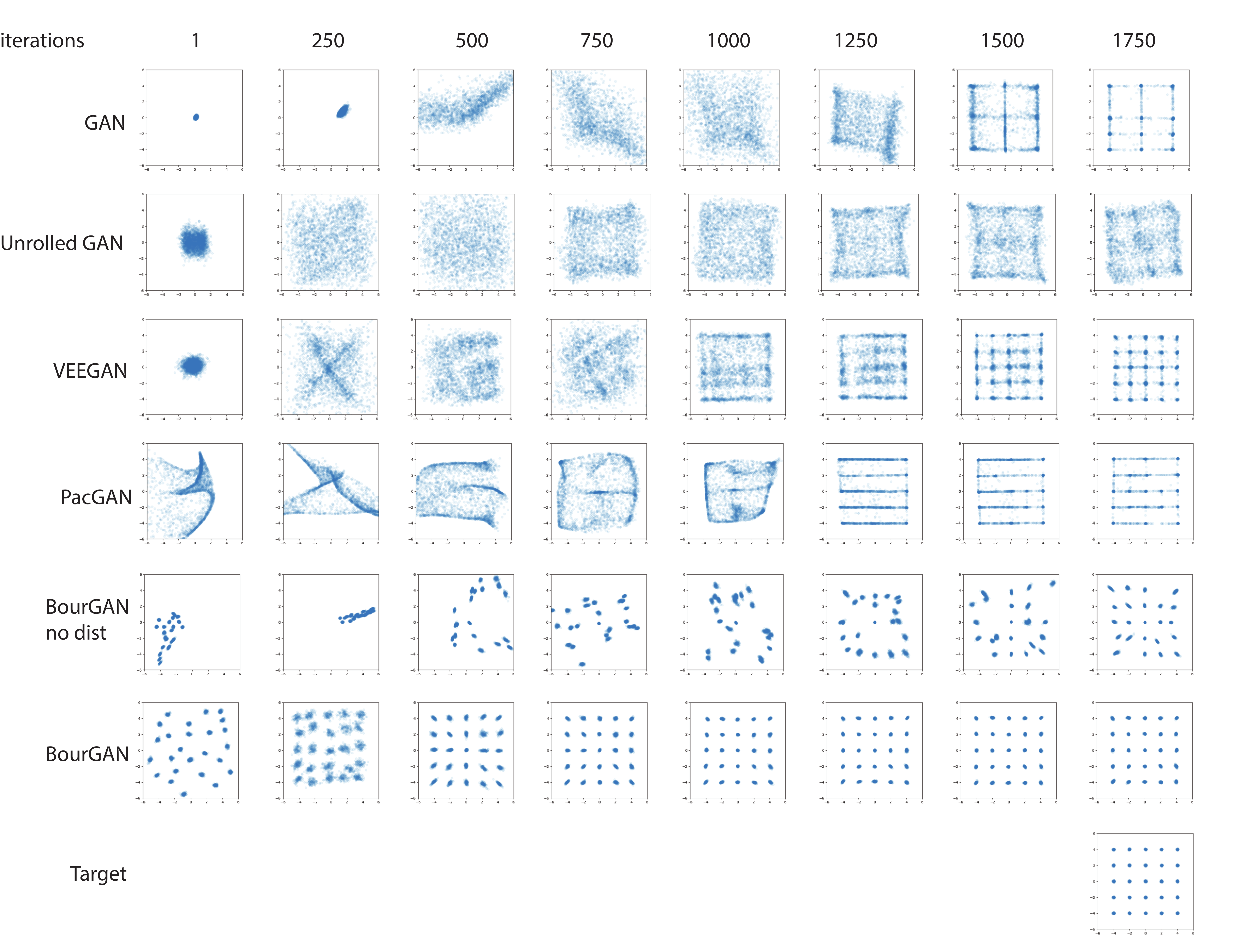}
\vspace{-10mm}
\caption{\textbf{How quickly do they converge?}
Our method outperforms other methods in terms of convergence rate in this example.
From left to right are the samples generated after the generators are trained over an increasing number of iterations.
The fifth row indicates the performance of Wasserstein GAN~\cite{arjovsky2017wasserstein}, although it is not particularly
designed for addressing mode collapse.
The sixth row reports the performance of BourGAN with standard GAN objective (i.e., no distance
preservation term~\eqref{eq:soft_constraint} is used). The seventh row indicate BourGAN with our
proposed objective function, which converges in the least number of iterations.
  }\label{fig:convergence}
\vspace{-1mm}
\end{figure}

Lastly, we examine how quickly these methods converges during the training
process. The results are reported in \figref{convergence}, where we also
include the results from our BourGAN but set $\beta$ in the
objective~\ref{eq:ourobj} to be zero. That is, we also test our method using
standard GAN objective function.  Figure~\ref{fig:convergence} shows that our
method with augmented objective converges the most quickly: The generator becomes
stable after 1000 iterations in this example, while others remain unstable even
after 1750 iterations. This result also empirically supports the necessity of
using the pairwise distance preservation term in the objective function.  We
attribute the faster convergence of our method to the fact that the
latent-space Gaussian mixture in our method encodes the structure of modes in
the data space and the fact that our objective function encourages the generator to
preserve this structure. 


\subsection{Evaluation on MNIST and Stacked MNIST}\label{sec:mnist_app}
In this section, we report the evaluation results on MNIST dataset. All MNIST
images are scaled to 32$\times$32 by bilinear interpolation.

\paragraph{Setup.}
Quantitative evaluation of GANs is known to be challenging,
because the implicit distributions of real datasets are hard, if not impossible, to obtain.
For the same reason, quantification of mode collapse is also hard for real datasets,
and no widely used evaluation protocol has been established.
We take an evaluation approach that has been used in a number of existing
GAN variants~\cite{borji2018pros, srivastava2017veegan, metz2016unrolled}:
we use a third-party trained classifier to classify the generated samples into specific modes,
and thereby estimate the generator's mode coverage~\cite{salimans2016improved}.



\paragraph{Classifier distance.}
A motivating observation of our method is that the structure of modes depends on
a specific choice of distance metric (recall \figref{pdd_tsne}). The widely used distance metrics on images
(such as the pixel-wise $\ell_2$ distance and Earth Mover's distance) may not necessarily produce
interpretable mode structures.
Here we propose to use the \textit{Classifier Distance} metric defined as
\begin{equation}\label{eq:cd}
d_{\text{classifier}}(x_i, x_j) = \| P(x_i)-P(x_j) \|_2,
\end{equation}
where $P(x_i)$ is the softmax output vector of a pre-trained
classification network, and $x_i$ represents an input image.
Adding a third-party trained classifier turns the task of training
generative models semi-supervised~\cite{mirza2014conditional}.
Nevertheless, Eq.~\eqref{eq:cd} is a highly complex distance metric,
serving for the purpose of testing our method with an ``unconventional''
metric. It is also meant to show that a properly chosen metric can produce interpretable modes.



\paragraph{Visualization of embeddings.}
After we apply our metric embedding algorithm with different distance metrics on MNIST images,
we obtain a set of vectors in $\ell_2$ space.
To visualize these vectors in 2D,
we use t-SNE \cite{maaten2008visualizing},
a nonlinear dimensionality reduction technique well-suited for visualization of
high-dimensional data in 2D or 3D.  Although not fully accurately, this
visualization shreds light on how (and where) data points are located in the
latent space (see \figref{pdd_tsne}).

\paragraph{MNIST experiment.}
First, we verify that our pre-training step (described in \appref{pretrain})
indeed accelerates the training process, as illustrated in \figref{pretrain_app}.

\begin{figure}
\centering
\includegraphics[width=0.99\textwidth]{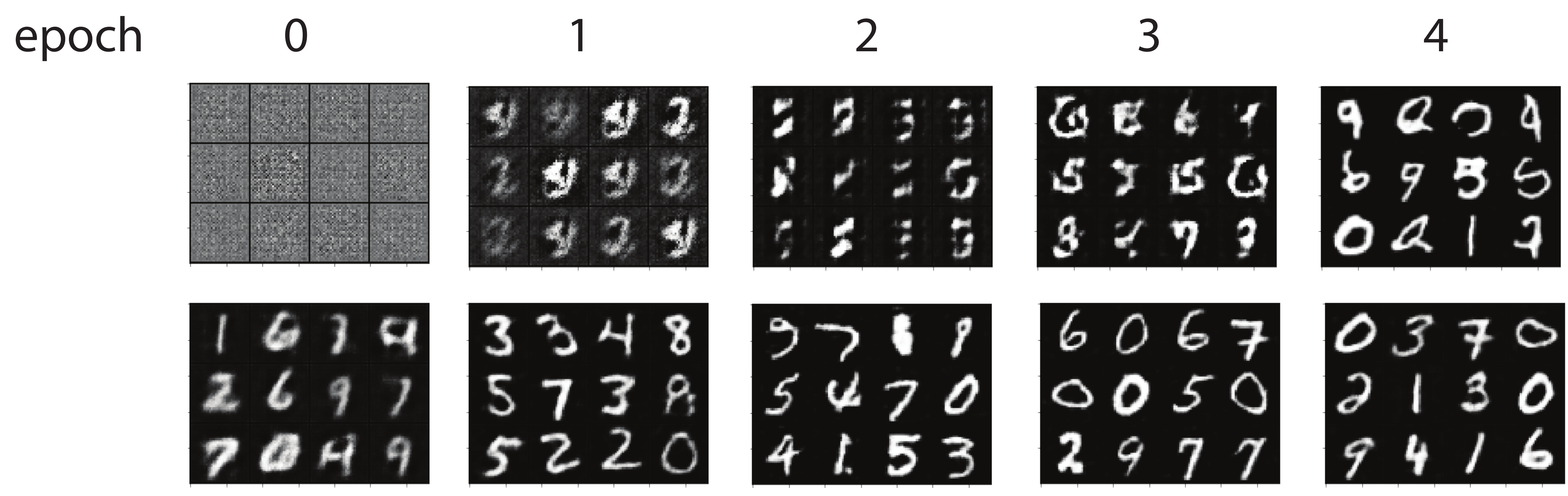}
\vspace{-1mm}
\caption{\textbf{Efficacy of pre-training.} (Top) BourGAN without pre-training. (Bottom) BourGAN with pre-training.
With the pre-training step, the GAN converges faster, and the generator network
produces better-quality results in each epoch.}\label{fig:pretrain_app}
\vspace{-1mm}
\end{figure}


Next, we evaluate the quality of generated samples using different distance metrics.
One widely used evaluation score is the inception score~\cite{gulrajani2017improved} that
measures both the visual quality and diversity of generated samples. However,
as pointed out by~\cite{che2016mode}, a generative model can produce a high
inception score even when it collapses to a visually implausible sample.
Furthermore, we would like to measure the visual quality and diversity
separately rather than jointly, to understand the performance of our
method in each of the two aspects under different metrics.
Thus, we choose to use entropy, defined as $E(x) = -\sum_{i=0}^{9} p(y=i|x)\log
p(y=i|x)$, as the score to measure the quality of the generated sample $x$, where $p(y=i|x)$
is the probability of labeling the input $x$ as the digit $i$ by the pre-trained
classifier.  The rationale here is that a high-quality sample often produces a
low entropy through the pre-trained classifier.

We compare DCGAN with BourGAN using this score.
Since our method can incorporate different distance metrics, we consider two of them:
BourGAN using $\ell_2$ distance and BourGAN using the aforementioned classifier distance.
For a fair comparison, the three GANS (i.e., DCGAN, BourGAN ($\ell_2$), and BourGAN (classifier))
all use the same number of dimensions ($k=55$) for the latent space and the same
network architecture.  For each type of GANs, we randomly generate 5000 samples
to evaluate the entropy scores, and the results are reported in \figref{mnist_stats}.
We also compute the KL divergence between the generated distribution and the data distribution,
following the practice of \cite{metz2016unrolled, maaten2008visualizing}.
The KL divergence for DCGAN, BourGAN ($\ell_2$) and BourGAN (classifier) are 0.116, 0.104, and 0.012, respectively.

A well-trained generator is expected to produce a relatively uniform
distribution across all 10 digits.  Our experiment suggests that both BourGAN ($\ell_2$)
and BourGAN (classifier) generate better-quality
samples in comparison to DCGAN, as they both produce lower entropy scores (\figref{mnist_stats}-right).
Yet, BourGAN (classifier) has a lower KL divergence compared to BourGAN ($\ell_2$),
suggesting that the classifier distance is a better metric in this
case to learn mode diversity. Although a pre-trained classifier may not always be available in real world
applications, here we demonstrate that some metric might be preferred over others
depending on the needs, and our method has the flexibility to use different metrics.

Lastly, we show that interpretable modes can be learned when a proper distance
metric is chosen.  Figure~\ref{fig:interpretable_mode} shows the generated
images when sampling around individual vectors in latent space. The BourGAN generator
trained with $\ell_2$ distance tends to produce images that are close to each other under
$\ell_2$ measure, while the generator trained with classifier distance tends to
produce images that are in the same class, which is more interpretable.

\begin{figure}
\centering
\includegraphics[width=0.99\textwidth]{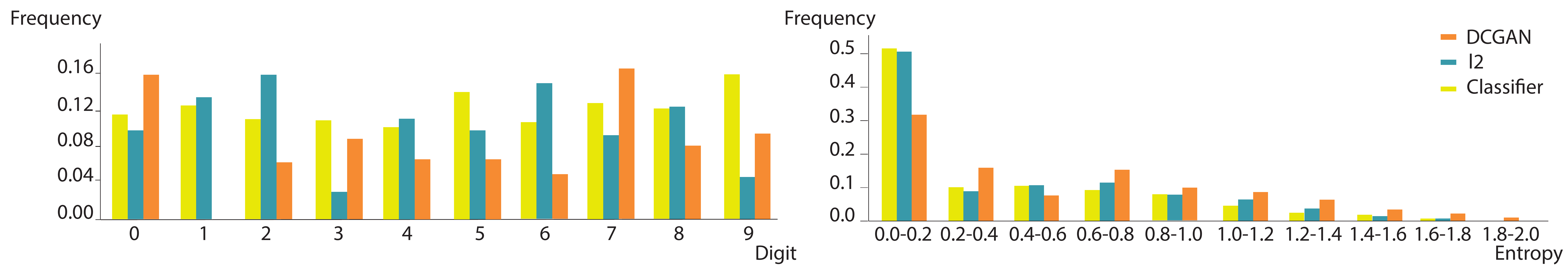}
\vspace{-1mm}
\caption{\textbf{MNIST dataset with different distance metrics}.
\textbf{(left)}
We plot the distribution of digits generated by DCGAN in orange, BourGAN ($\ell_2$) in green,
and BourGAN (classifier) in yellow. The generated images from those GANs are classified
using a pre-trained classifier.
This plot shows that the classifier distance produces samples that are most uniformly distributed
across all 10 digits. DCGAN fails to capture the mode of digital ``1'',
while BourGAN ($\ell_2$) generates fewer samples for the modes in ``3'' and ``9''.
\textbf{(right)} Entropy
distribution of generated samples using three GANs.
A lower entropy value indicates better image quality.
This plot suggests that our method with both $\ell_2$ and classifier distance metrics
produces higher-quality MNIST images than the DCGAN.}\label{fig:mnist_stats}
\vspace{-1mm}
\end{figure}

\begin{figure}
	\centering
\includegraphics[width=0.99\textwidth]{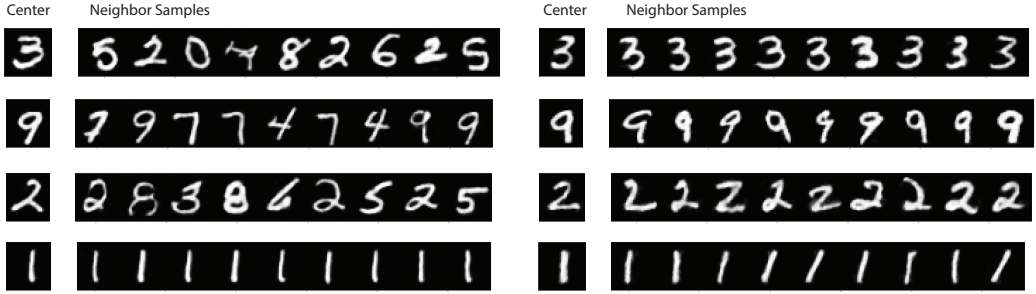}
\vspace{-1mm}
\caption{\textbf{Interpretable modes.}
Using BourGAN, we first randomly generate four samples and use their latent
vectors as four centers in latent space.  We then sample nine latent vectors in a
hypersphere of each center, and use these vectors to generate MNIST images. The
hypersphere has a radius of 0.1
\textbf{(Left)} BourGAN ($\ell_2$) generates samples that are close to others in the same hypersphere
in $\ell_2$ space. But the samples can be
visually distinct from each other, representing different digits.
Note that under $\ell_2$ distance, digit ``1'' are separated out (the fourth row on the left).
It is interesting to recall the bottom-left subfigure of \figref{pdd_tsne}, and realize that this resonates 
with that subfigure in which data items of digit ``1'' are clustered as a separated mode in $\ell_2$ metric.
\textbf{(Right)} BourGAN (classifier) is trained with the classifier distance, which 
tends to cluster together images that represent the same type of digits.
As a result, the generated samples tend to represent the same digits as their respective centers.
Thus, the modes captured by BourGAN (classifier) is more interpretable. In this case, each mode corresponds to
a different digit.
}\label{fig:interpretable_mode}
\vspace{-1mm}
\end{figure}

\paragraph{Tests on Stacked MNIST.}

Similar to the evaluation methods in Mode-regularized GANs \cite{che2016mode}, Unrolled GANs
\cite{metz2016unrolled}, VEEGAN \cite{srivastava2017veegan} and PacGAN
\cite{lin2017pacgan}, we test BourGAN with $\ell_2$ distance metric
on an augmented MNIST dataset. By encapsulating three randomly selected MNIST images into three color channels,
we construct a new dataset of 100,000 images, each of which has a dimension of 32$\times$32$\times$3.
In the end, we obtain 10$\times$10$\times$10 = 1000 distinct classes.
We refer to this dataset as the \emph{stacked MNIST} dataset. In this experiment,
we will treat each of the 1000 classes of images as an individual mode.

\begin{table}[t]
\centering
\scalebox{0.99}{
\begin{tabular}{@{}lllllllllll@{}}
\whline{1.0pt}
& \multicolumn{2}{c}{D is 1/4 size of G}                  & \multicolumn{2}{c}{D is 1/2 size of G}  & \multicolumn{2}{c}{D is same size as G}  \\ \cmidrule(l){2-3}\cmidrule(l){4-5}\cmidrule(l){6-7}
& {\begin{tabular}[c]{@{}c@{}}\# class covered \\ (max 1000)\end{tabular}}     & KL   							& {\begin{tabular}[c]{@{}c@{}}\# class covered \\ (max 1000)\end{tabular}}    & KL & {\begin{tabular}[c]{@{}c@{}}\# class covered \\ (max 1000)\end{tabular}}    & KL     \\ 
\whline{0.7pt}
DCGAN    &     	92.2				& 5.02								& 367.7				& 4.87						& 912.3			& 0.65 \\
BourGAN  &      715.2  				& 1.84								&		936.1		&		0.61				&		1000.0	& 0.08 \\
\whline{1.0pt}
\end{tabular}}
\caption{\textbf{Mode coverage} on stacked MNIST Dataset. Results are averaged over 10 trials}\label{tab:stacked_mnist}
\end{table}

As reported in~\cite{metz2016unrolled}, even regular GANs can learn all 1000
modes if the discriminator size is sufficiently large.
Thus, we evaluate our method by setting the discriminator's size to be
$\nicefrac{1}{4}\times$, $\nicefrac{1}{2}\times$, and $1\times$ of the generator's size, respectively.
We measure the number of modes captured by our method as well as by DCGAN, and the
KL divergence between the generated distribution of modes and the expected true
distribution of modes (i.e., a uniform distribution over the 1000 modes).  
\tabref{stacked_mnist} summarizes our results.
In Table 2 and 3 of their paper, Lin et al.~\cite{lin2017pacgan}
reported results on similar experiments, although we note that it is hard to directly compare our \tabref{stacked_mnist} 
with theirs, because their detailed network setup and the third-part classifier may differ from ours.
We summarize our network structures in \tabref{structure} and \ref{tab:structure2}.
During training, we use Adam optimization with a learning rate of $10^{-4}$,
and set $\beta_1=0.5$ and $\beta_2=0.999$ with a mini-batch size of 128.

Additionally, in \figref{qualitative_stacked_mnist}
we show a qualitative comparison between our method and DCGAN on this dataset.

\begin{table}[b]
\small
\centering
\begin{tabular}{lllllll}
\whline{1.0pt}
	 & layer 		   & output size & kernel size & stride & BN  & activation function  \\
\whline{0.7pt}
	 & input (dim 55)  &  55$\times$1$\times$1 	 &              &        &     &      \\
	 & Transposed Conv &  512$\times$4$\times$4 	 & 4 		& 1 	 & Yes & ReLU  \\
	 & Transposed Conv &  256$\times$8$\times$8 	 & 4  		& 2 	 & Yes & ReLU  \\
	 & Transposed Conv &  128$\times$16$\times$16    & 4 		& 2 	 & Yes & ReLU  \\
	 & Transposed Conv &  channel$\times$32$\times$32  & 4 		& 2 	 & No  & Tanh  \\
\whline{1.0pt}
	\end{tabular}
        \caption{\textbf{Network structure} for generator. channel=3 for Stacked
        MNIST and channel=1 for MNIST.} \label{tab:structure}
\end{table}

\begin{table}[t]
\small
	\centering
	\begin{tabular}{lllllll}
\whline{1.0pt}
	 & layer 		   & output size & kernel size & stride & BN  & activation function  \\
\whline{0.7pt}
	 & input (dim 55)  &  channel$\times$32$\times$32  &  		& 	     &    &  	   \\
	 & Conv            &  256$\times$16$\times$16 	 & 4 		& 2 	 & No & LeakyReLU(0.2)  \\
	 & Conv            &  256$\times$8$\times$8 	      & 4  		& 2 	 & Yes & LeakyReLU(0.2) \\
	 & Conv            &  128$\times$4$\times$4       & 4 		& 2 	 & Yes & LeakyReLU(0.2)  \\
	 & Conv            &  channel$\times$1$\times$1  & 4 		& 1 	     & No  & Sigmoid  \\
\whline{1.0pt}
	\end{tabular}
	\caption{\textbf{Network structure} for discriminator.} \label{tab:structure2}
\end{table}

\begin{figure}[t]
	\centering
	  \includegraphics[width=0.99\textwidth]{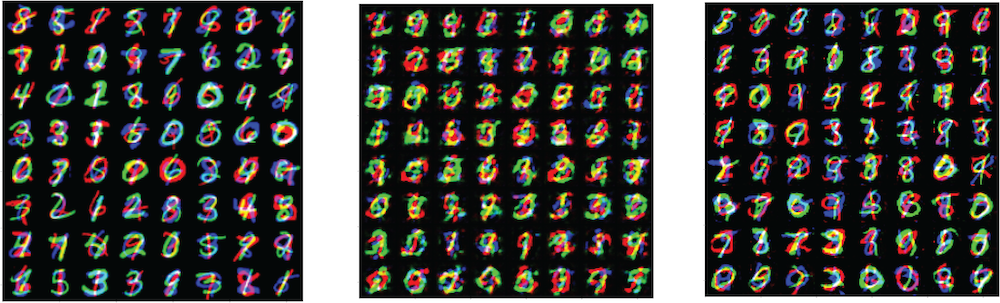}
\vspace{-1mm}
\caption{Qualitative results on stacked MNIST dataset.
\textbf{(Left)} Samples from real data distribution. \textbf{(Middle)} Samples
generated by DCGAN. \textbf{(Right)} Samples generated by BourGAN.
In all three GANs, discriminator network has a size $\nicefrac{1}{4}\times$ of the generator.
DCGAN starts to generate collapsed results,
while BourGAN still generates plausible results.}\label{fig:qualitative_stacked_mnist}
\vspace{-1mm}
\end{figure}

\subsection{More Qualitative Results} \label{sec:qual}
We also test our algorithm on other popular dataset, including
CIFAR-10~\cite{krizhevsky2009learning} and Fashion-MNIST
\cite{xiao2017fashion}.
Figure~\ref{fig:qualitative_cifar} and \ref{fig:qualitative_fashion} illustrate our results on these datasets.
\begin{figure}[t!]
\centering
\includegraphics[width=0.5\textwidth]{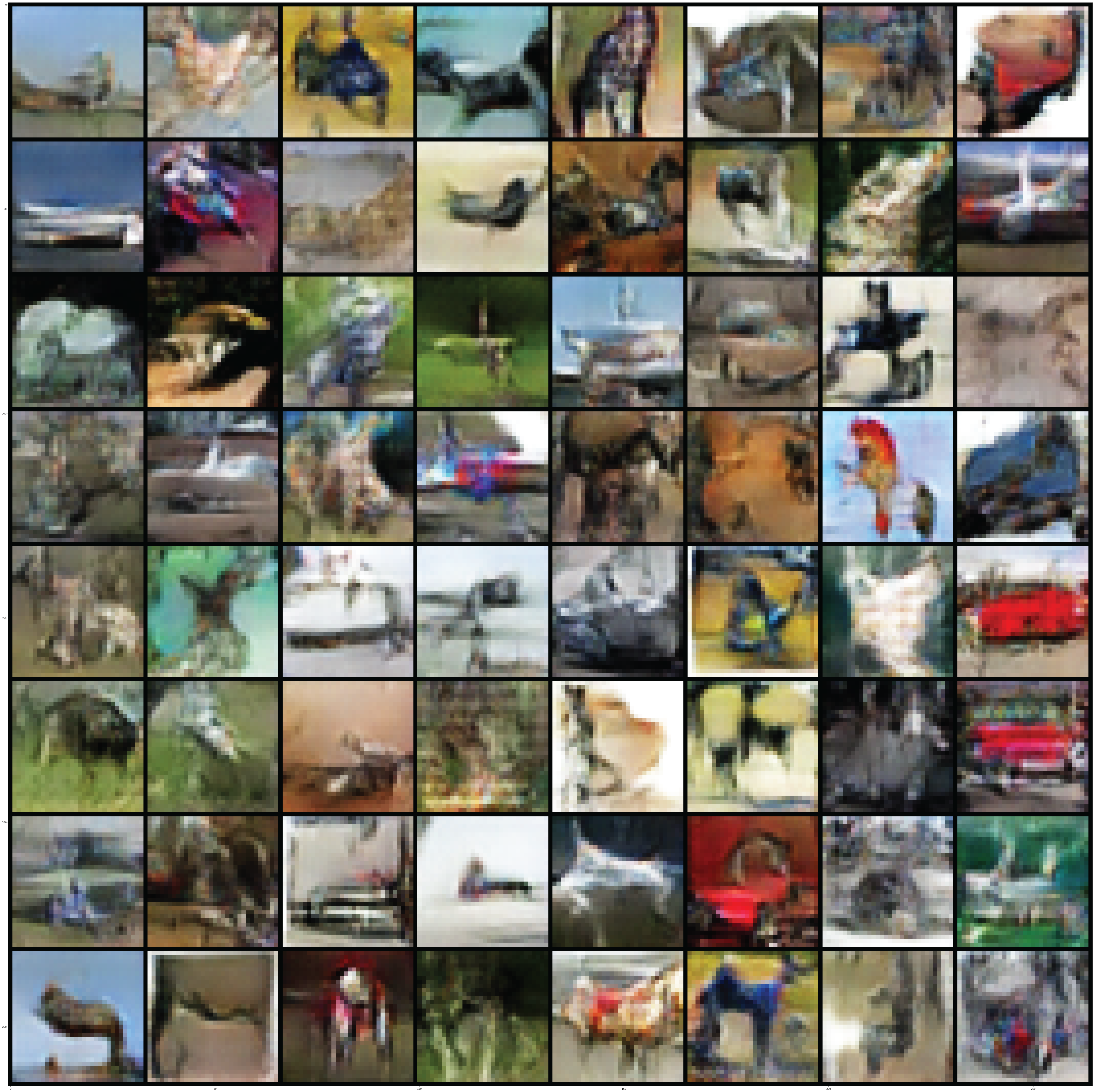}
\vspace{-1mm}
\caption{Qualitative results on CIFAR-10.}\label{fig:qualitative_cifar}
\vspace{-1mm}
\end{figure}

\begin{figure}[t!]
\centering
\includegraphics[width=0.5\textwidth]{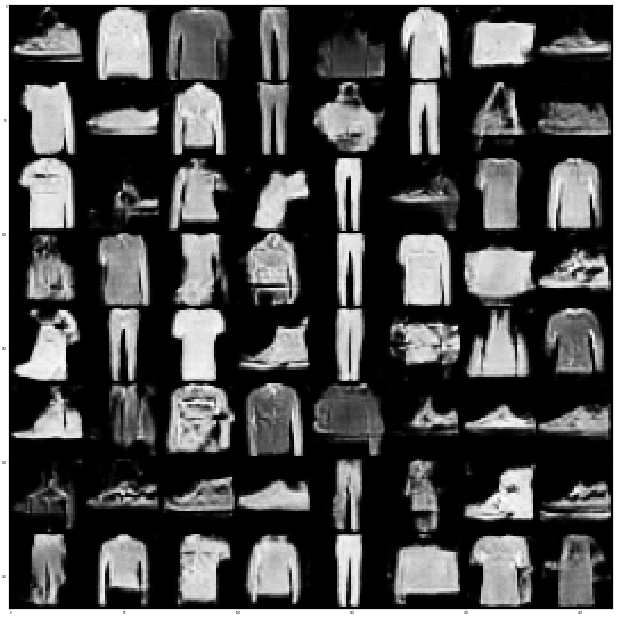}
\vspace{-1mm}
\caption{Qualitative results on Fashion-MNIST.}\label{fig:qualitative_fashion}
\vspace{-1mm}
\end{figure}

\newpage
\section{Proofs of the Theorems in Section~\ref{sec:theory}}

\subsection{Notations and Preliminaries}
Before we delve into technical details, we first review some notation
and fundamental tools in the theoretical analysis:
We use $\mathbf{1}(\mathcal{E})$ to denote an indicator variable on the event
$\mathcal{E}$, i.e., if $\mathcal{E}$ happens, then $\mathbf{1}(\mathcal{E})=1$,
otherwise, $\mathbf{1}(\mathcal{E})=0$.

The following lemma gives a concentration bound on independent random variables.
\begin{lemma}[Bernstein Inequality]\label{lem:bernstein}
Let $X_1,X_2,\cdots,X_n$ be $n$ independent random variables.
Suppose that $\forall i\in[n],|X_i-\E(X_i)|\leq M$ almost surely.
Then, $\forall t>0,$
\begin{align*}
\Pr\left(\left|\sum_{i=1}^n X_i - \sum_{i=1}^n \E(X_i)\right|>t \right)\leq 2\exp\left(-\frac{\frac{1}{2}t^2}{\sum_{i=1}^n \var(X_i)+\frac{1}{3}Mt}\right).
\end{align*}
\end{lemma}

The next lemma states that given a complete graph with a power of $2$ number of vertices, the edges can be decomposed into perfect matchings.
\begin{lemma}\label{lem:K_m_decomposition}
Given a complete graph $G=(V,E)$ with $|V|=m$ vertices, where $m$ is a power of $2$. Then, the edge set $E$ can be decomposed into $m-1$ perfect matchings.
\end{lemma}
\begin{figure}[h!]
  \centering
  \includegraphics[width=0.35\textwidth]{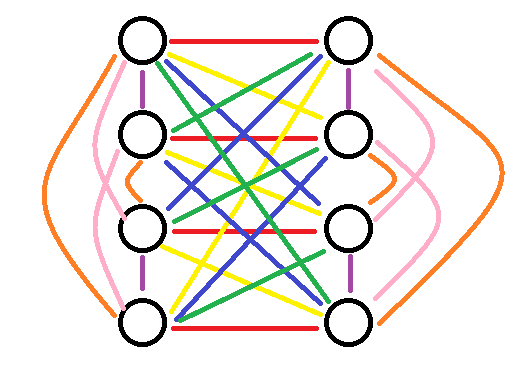} 
  \caption{An $8$-vertices complete graph can be decomposed into $7$ perfect matchings}
\end{figure}

\begin{proof}
Our proof is by induction.
The base case has $m=1$.
For the base case, the claim is obviously true.
Now suppose that the claim holds for $m/2.$
Consider a complete graph $G=(V,E)$ with $m$ vertices, where $m$ is a power of $2$.
We can partition vertices set $V$ into two vertices sets $A,B$ such that $|A|=|B|=m/2.$
The edges between $A$ and $B$ together with vertices $A\cup B=V$ compose a complete bipartite graph.
Thus, the edges between $A$ and $B$ can be decomposed into $m/2$ perfect matchings.
The subgraph of $G$ induced by $A$ is a complete graph with $m/2$ vertices.
By our induction hypothesis, the edge set of the subgraph of $G$ induced by $A$ can be decomposed into $m/2-1$ perfect matchings in that induced subgraph.
Similarly, the edge set of the subgraph of $G$ induced by $B$ can be also decomposed into $m/2-1$ perfect matchings in that induced subgraph.
Notice that any perfect matching in the subgraph induced by $A$ union any perfect matching in the subgraph induced by $B$ is a perfect matching of $G$.
Thus, $E$ can be decomposed into $m/2+m/2-1=m-1$ perfect matchings.
\end{proof}

\subsection{Proof of Theorem~\ref{thm:range_of_lambda}}\label{sec:thm3_proof}
In the following, we formally restate the theorem.
\begin{theorem}\label{thm:estimate_range_of_lambda}
Consider a metric space $(\M,\dist).$
Let $\mathcal{X}$ be a distribution over $\M$ which satisfies $\Pr_{a,b\sim \mathcal{X}}(a\not=b)\geq 1/2.$
Let $x_1,x_2,\cdots,x_n$ be $n$ i.i.d. samples drawn from $\mathcal{X}$.
Let $\lambda=\min_{i\in[n-1]:x_i\ne x_{i+1}}\dist(x_{i},x_{i+1}),\Lambda=\max_{i\in[n-1]:x_i\ne x_{i+1}}\dist(x_{i},x_{i+1}).$
For any given parameters $\delta\in(0,1),\gamma\in(0,1),$ if $n\geq C/(\delta\gamma)$ for some sufficiently large constant $C>0$, then with probability at least $1-\delta,$ $\Pr_{a,b\sim \mathcal{X}}(\dist(a,b)\in[\lambda,\Lambda]\mid \lambda,\Lambda)\geq \Pr_{a,b\sim\mathcal{X}}(a\not=b)-\gamma.$
\end{theorem}
\begin{proof}
Without of loss of generality, we assume $n$ is an even number. 
Let $\lambda'=\min_{i\in[n/2]:x_{2i-1}\not=x_{2i}}\dist(x_{2i-1},x_{2i}), \Lambda'=\max_{i\in[n/2]:x_{2i-1}\not=x_{2i}}\dist(x_{2i-1},x_{2i})$,
and $P$, $Q$ be two i.i.d. random variables with distribution $\mathcal{X}$.
Then $(x_1,x_2),(x_3,x_4),\cdots,(x_{n-1},x_n)$ are $n/2$ i.i.d. samples drawn from the same distribution as $(P,Q).$
Let $t=|\{j\in[n/2]\mid x_{2j-1}\not=x_{2j}\}|.$
Suppose $p$ is the probability, $p=\Pr(P\not=Q)$, then we have the following relationship.
\begin{align}
&\Pr_{P,Q,x_1,\cdots,x_m\sim \mathcal{D}}\left(\dist(P,Q)<\lambda'\vee\dist(P,Q)>\Lambda'\mid P\not=Q\right)\notag\\
=~&\Pr_{P,Q,x_1,\cdots,x_m\sim \mathcal{D}}\left(\dist(P,Q)<\lambda'\vee\dist(P,Q)>\Lambda'\mid P\not=Q,t\geq pn/2\right)\cdot \Pr(t\geq pn/2\mid P\not=Q)\\
&+\Pr_{P,Q,x_1,\cdots,x_m\sim \mathcal{D}}\left(\dist(P,Q)<\lambda'\vee\dist(P,Q)>\Lambda'\mid P\not=Q,t\leq pn/2\right)\cdot \Pr(t\leq pn/2\mid P\not=Q)\\
\leq~&\Pr(\dist(P,Q)<\lambda'\vee\dist(P,Q)>\Lambda'\mid P\not=Q,t\geq pn/2)+\Pr(t< pn/2)\notag\\
\leq~&\frac{2}{1+pn/2}\notag+\Pr(t< pn/2)\\
\leq~&\frac{2}{1+n/4}+2^{-\Theta(n)}\notag\\
\leq~&\frac{4}{1+n/4}\leq 16/n\label{eq:number_samples}
\end{align}
where the first inequality follows by that probability is always upper bounded by $1$, the second inequality follows by symmetry of $(P,Q)$ and $(x_{2j-1},x_{2j}),$ the third inequality follows by $p\geq 1/2$ and the Chernoff bound, the forth inequality follows by that $n$ is sufficiently large.


Notice that if with probability greater than $\delta,$ $\Pr(\dist(P,Q)<\lambda'\text{~or~}\dist(P,Q)>\Lambda'\mid \lambda',\Lambda')>1-p+\gamma,$ then we have with probability greater than $\delta,$
\begin{align*}
1-p+\gamma<~&\Pr(\dist(P,Q)<\lambda'\vee\dist(P,Q)>\Lambda'\mid \lambda',\Lambda')\\
=~&\Pr(\dist(P,Q)<\lambda'\vee\dist(P,Q)>\Lambda'\mid \lambda',\Lambda',P\not=Q)\cdot \Pr(P\not=Q)+\Pr(P=Q)\\
=~&\Pr(\dist(P,Q)<\lambda'\vee\dist(P,Q)>\Lambda'\mid \lambda',\Lambda',P\not=Q)\cdot p+1-p
\end{align*}
which implies that with probability greater than $\delta,$ $\Pr(\dist(P,Q)<\lambda'\text{~or~}\dist(P,Q)>\Lambda'\mid \lambda',\Lambda',P\not=Q)>\gamma/p\geq \gamma.$
Then we have $\Pr(\dist(P,Q)<\lambda'\text{~or~}\dist(P,Q)>\Lambda'\mid P\not=Q)> \delta\gamma\geq 16/n$ which contradicts to
Equation~\eqref{eq:number_samples}.

Notice that $\lambda\leq \lambda'$ and $\Lambda\geq \Lambda',$ we complete the proof.
\end{proof}

\subsection{Proof of Theorem~\ref{thm:main_small_sample_enough}}\label{sec:thm4_proof}
We restate the theorem in the following formal way.
\begin{theorem}\label{thm:small_set_good_approx}
Consider a metric space $(\M,\dist).$
Let $\mathcal{X}$ be a distribution over $\M.$
Let $\lambda,\Lambda$ be two parameters such that $0<2\lambda\leq \Lambda.$
Let $\mathcal{P}$ be the $(\lambda,\Lambda)-$LPDD of $\mathcal{X}.$
Let $y_1,y_2,\cdots,y_m$ be $m$ i.i.d. samples drawn from distribution $\mathcal{X},$ where $m$ is a power of $2$.
Let $\mathcal{P}'$ be the $(\lambda,\Lambda)-$LPDD of the uniform distribution on $Y$.
Let $\gamma=\Pr_{a,b\sim\mathcal{X}}(\dist(a,b)\in [\lambda,\Lambda]).$
Given $\delta\in(0,1),\varepsilon\in(0,\log(\Lambda/\lambda)),$ if $m\geq C\cdot\frac{\log^4(\Lambda/\lambda)}{\varepsilon^4\gamma^4}\cdot \log\left(\frac{\log(\Lambda/\lambda)}{\min(\varepsilon,1)\gamma\delta}\right)$ for some sufficiently large constant $C>0,$ then with probability at least $1-\delta,$ we have $W(\mathcal{P},\mathcal{P}')\leq \varepsilon.$
\end{theorem}
\begin{proof}
Suppose $m\geq C\cdot\frac{\log^4(\Lambda/\lambda)}{\varepsilon^4\gamma^4}\cdot \log\left(\frac{\log(\Lambda/\lambda)}{\min(\varepsilon,1)\gamma\delta}\right)$ for some sufficiently large constant $C>0.$
Let $\mathcal{U}$ be a uniform distribution over $m$ samples $\{y_1,y_2,\cdots,y_m\}.$
Let $\varepsilon_0=\varepsilon/2,i_0=\lfloor\log_{1+\varepsilon_0}\lambda\rfloor,i_1=\lceil\log_{1+\varepsilon_0}\Lambda\rceil,$ and $\alpha=(1+\varepsilon_0).$
Let $I$ be the set $\{i_0,i_0+1,i_0+2,\cdots,i_1-1,i_1\}.$
Then we have $|I|\leq \log(\Lambda/\lambda)/\varepsilon_0.$
Since $\mathcal{P},\mathcal{P}'$ are $(\lambda,\Lambda)-$LPDD of $\mathcal{X}$ and uniform distribution on $Y$ respectively, we have 
\begin{align*}
&W(\mathcal{P},\mathcal{P}')\\
\leq~& \sum_{i=i_0}^{i_1} \min\left(~\Pr_{p\sim \mathcal{P}}(~p\in[i,i+1)\cdot \log\alpha~),~\Pr_{p'\sim\mathcal{P}'}(~p'\in[i,i+1)\cdot \log\alpha~)~\right)\cdot\log\alpha\\
&+\sum_{i=i_0}^{i_1} \left|\Pr_{p\sim\mathcal{P}}(~p\in[i,i+1)\cdot \log\alpha~)-\Pr_{p'\sim \mathcal{P}'}(~p'\in[i,i+1)\cdot \log\alpha~)\right|\cdot \log(\Lambda/\lambda)\\
\leq~& \varepsilon_0+\sum_{i=i_0}^{i_1} \left|\Pr_{p\sim \mathcal{P}}(~p\in[i,i+1)\cdot \log\alpha~)-\Pr_{p'\sim\mathcal{P'}}(~p'\in[i,i+1)\cdot \log\alpha~)\right|\cdot \log(\Lambda/\lambda).
\end{align*}
Thus, to prove $W(\mathcal{P},\mathcal{P}')\leq\varepsilon= 2\varepsilon_0,$ it suffices to show that
\begin{equation}
  \begin{aligned}
\forall i\in I, \left|\Pr_{p\in\mathcal{P}}(p\in [i,i+1)\cdot
    \log\alpha)-\Pr_{p'\sim\mathcal{P'}}(p'\in [i,i+1)\cdot
        \log\alpha)\right| &\leq \frac{\varepsilon_0}{|I|\cdot\log\left(\Lambda/\lambda\right)} \\
        & \leq \frac{\varepsilon_0^2}{2\log^2(\Lambda/\lambda)}.\label{eq:final}
  \end{aligned}
\end{equation}
For an $i\in I,$ consider $\Pr_{p\in\mathcal{P}}(p\in [i,i+1)\cdot \log\alpha),$ we have
\begin{align*}
\Pr_{p\in\mathcal{P}}(p\in [i,i+1)\cdot \log\alpha)=\frac{\underset{a,b\sim \mathcal{X}}{\Pr}(\dist(a,b)\in[\alpha^i,\alpha^{i+1}))}{\underset{a,b\sim\mathcal{X}}\Pr(\dist(a,b)\in[\lambda,\Lambda])}.
\end{align*}
Consider $\Pr_{p'\sim \mathcal{P}'}(p'\in [i,i+1)\cdot \log\alpha),$ we have
\begin{align}
&\Pr_{p'\sim \mathcal{P}'}(p'\in [i,i+1)\cdot \log\alpha)\notag\\
=&\underset{a',b'\sim \mathcal{U}}{\Pr}(\dist(a',b')\in[\alpha^i,\alpha^{i+1})\mid \dist(a',b')\in[\lambda,\Lambda])\notag\\
=&\frac{1/(m(m-1))\cdot\sum_{j\not=k}\mathbf{1}(\dist(y_j,y_k)\in [\alpha^i,\alpha^{i+1}))}{1/(m(m-1))\cdot \sum_{j\not =k}\mathbf{1}(\dist(y_j,y_k)\in[\lambda,\Lambda])},\label{eq:new_dis}
\end{align}
where $\mathbf{1}(\cdot)$ is an indicator function. 
In the following parts, we will focus on giving upper bounds on the difference
\begin{align}
\left|\frac{\sum_{j\not=k}\mathbf{1}(\dist(y_j,y_k)\in [\alpha^i,\alpha^{i+1}))}{m(m-1)}-\underset{a,b\sim \mathcal{X}}{\Pr}\left(\dist(a,b)\in[\alpha^i,\alpha^{i+1})\right)\right|\label{eq:bound_1}
\end{align}
and the difference
\begin{align}
\left|\frac{\sum_{j\not =k}\mathbf{1}(\dist(y_j,y_k)\in[\lambda,\Lambda])}{m(m-1)}-\underset{a,b\sim\mathcal{X}}\Pr(\dist(a,b)\in[\lambda,\Lambda])\right|.\label{eq:bound_2}
\end{align}

Now we look at a fixed $i\in I.$
Let $S$ be the set of all possible pairs $(y_j,y_k),$ i.e. $S=\{(y_j,y_k)\mid j,k\in[m],j\not=k\}.$
According to Lemma~\ref{lem:K_m_decomposition}, $S$ can be decomposed into $2(m-1)$ sets $S_1,S_2,\cdots,S_{2(m-1)}$ each with size $m/2,$ i.e. $S=\bigcup_{l=1}^{2(m-1)}S_l,\forall l\in[2(m-1)],|S_l|=m/2,$ and furthermore, $\forall l\in[2(m-1)],j\in[m],$ $y_j$ only appears in exactly one pair in set $S_l.$
It means that $\forall l\in[2(m-1)],$ $S_l$ contains $m/2$ i.i.d. random samples drawn from $\mathcal{X}\times \mathcal{X},$ where $\mathcal{X}\times\mathcal{X}$ is the joint distribution of two i.i.d. random samples $a,b$ each with marginal distribution $\mathcal{X}.$
For $l\in[2(m-1)],$ by applying Bernstein inequality (see Lemma~\ref{lem:bernstein}), we have:
\begin{align*}
&\Pr\left(\left| \frac{\sum_{(x,y)\in S_l}\mathbf{1}(\dist(x,y)\in [\alpha^i,\alpha^{i+1}))}{m/2}-\underset{a,b\sim \mathcal{X}}{\Pr}\left(\dist(a,b)\in[\alpha^i,\alpha^{i+1})\right)\right| >\frac{\gamma\varepsilon_0^2}{8\log^2(\Lambda/\lambda)}\right)\\
=~&\Pr\left(\left| \sum_{(x,y)\in S_l}\mathbf{1}(\dist(x,y)\in [\alpha^i,\alpha^{i+1}))-\sum_{(x,y)\in S_l}\underset{a,b\sim \mathcal{X}}{\Pr}\left(\dist(a,b)\in[\alpha^i,\alpha^{i+1})\right)\right| >\frac{m\cdot \gamma\varepsilon_0^2}{4\log^2(\Lambda/\lambda)}\right)\\
\leq~&2\exp\left(-\frac{\frac{1}{32} \cdot m^2\cdot \gamma^2\varepsilon_0^4/\log^4(\Lambda/\lambda)}{m/2+m\cdot \gamma\varepsilon_0/\log^2(\Lambda/\lambda)\cdot 1/48}\right)\\
\leq~&2\exp\left(-\frac{\frac{1}{32} \cdot m^2\cdot \gamma^2\varepsilon_0^4/\log^4(\Lambda/\lambda)}{m/2+m/2}\right)\\
=~&2\exp\left(-\frac{1}{32} \cdot m\cdot \gamma^2\varepsilon_0^4/\log^4(\Lambda/\lambda)\right)\\
\leq~& \frac{\delta}{2}\cdot \frac{1}{2(m-1)|I|},
\end{align*}
where the first inequality follows by plugging $|S_l|=m/2$ i.i.d. random
variables $\mathbf{1}(\dist(x,y)\in [\alpha^i,\alpha^{i+1}))$ for all $(x,y)\in
S_l,$ $t=(m\cdot \gamma\varepsilon_0^2)/(4\log^2(\Lambda/\lambda))$ and $M=1$
into Lemma~\ref{lem:bernstein}, the second inequality follows by
$\gamma\varepsilon_0^2/\log^2(\Lambda/\lambda)\leq 1$, where recall $\gamma=\Pr_{a,b\sim\mathcal{X}}(\dist(a,b)\in [\lambda,\Lambda]).$
and the last inequality
follows by the choice of $m$ and $(m-1)\leq m,|I|\leq 2\log(\Lambda/\lambda)/\varepsilon_0.$
By taking union bound over all the sets $S_1,S_2,\cdots,S_{2(m-1)},$ with probability at least $1-\delta/2\cdot 1/|I|,$ we have $\forall l\in[2(m-1)],$
\begin{align*}
\left| \frac{\sum_{(x,y)\in S_l}\mathbf{1}(\dist(x,y)\in [\alpha^i,\alpha^{i+1}))}{m/2}-\underset{a,b\sim \mathcal{X}}{\Pr}\left(\dist(a,b)\in[\alpha^i,\alpha^{i+1})\right)\right| \leq\frac{\gamma\varepsilon_0^2}{8\log^2(\Lambda/\lambda)}.
\end{align*}
In this case, we have:
\begin{align*}
\small
\left| \sum_{l=1}^{2(m-1)}\sum_{(x,y)\in S_l}\frac{\mathbf{1}(\dist(x,y)\in [\alpha^i,\alpha^{i+1}))}{m/2}-2(m-1)\underset{a,b\sim \mathcal{X}}{\Pr}\left(\dist(a,b)\in[\alpha^i,\alpha^{i+1})\right)\right| \leq\frac{2(m-1)\gamma\varepsilon_0^2}{8\log^2(\Lambda/\lambda)}.
\end{align*}
Since $S=\bigcup_{l=1}^{2(m-1)}S_l=\{(y_j,y_k)\mid j,k\in[m],j\not=k\},$ we have
\begin{align*}
\left|\frac{\sum_{j\not=k}\mathbf{1}(\dist(y_j,y_k)\in [\alpha^i,\alpha^{i+1}))}{m(m-1)}-\underset{a,b\sim \mathcal{X}}{\Pr}\left(\dist(a,b)\in[\alpha^i,\alpha^{i+1})\right)\right|\leq \frac{\gamma\varepsilon_0^2}{8\log^2(\Lambda/\lambda)}.
\end{align*}
By taking union bound over all $i\in I,$ then with probability at least $1-\delta/2,$ $\forall i\in I,$ we have
\begin{align}
\left|\frac{\sum_{j\not=k}\mathbf{1}(\dist(y_j,y_k)\in [\alpha^i,\alpha^{i+1}))}{m(m-1)}-\underset{a,b\sim \mathcal{X}}{\Pr}\left(\dist(a,b)\in[\alpha^i,\alpha^{i+1})\right)\right|\leq \frac{\gamma\varepsilon_0^2}{8\log^2(\Lambda/\lambda)}.\label{eq:ub_1}
\end{align}
Thus, we have an upper bound on Equation~\eqref{eq:bound_1}.

Now, let us try to derive an upper bound on Equation~\eqref{eq:bound_2}.
Similar as in the previous paragraph, we let $S$ be the set of all possible pairs $(y_j,y_k),$ i.e. $S=\{(y_j,y_k)\mid j,k\in[m],j\not=k\}.$
$S$ can be decomposed into $2(m-1)$ sets $S_1,S_2,\cdots,S_{2(m-1)}$ each with size $m/2,$ i.e. $S=\bigcup_{l=1}^{2(m-1)}S_l,\forall l\in[2(m-1)],|S_l|=m/2,$ and furthermore, $\forall l\in[2(m-1)],j\in[m],$ $y_j$ only appears in exactly one pair in set $S_l.$
For $l\in[2(m-1)],$ by applying Bernstein inequality (see Lemma~\ref{lem:bernstein}), we have:
\begin{align*}
&\Pr\left(\left| \frac{\sum_{(x,y)\in S_l}\mathbf{1}(\dist(x,y)\in [\lambda,\Lambda])}{m/2}-\underset{a,b\sim \mathcal{X}}{\Pr}\left(\dist(a,b)\in[\lambda,\Lambda]\right)\right| >\frac{\gamma^2\varepsilon_0^2}{8\log^2(\Lambda/\lambda)}\right)\\
=~&\Pr\left(\left| \sum_{(x,y)\in S_l}\mathbf{1}(\dist(x,y)\in [\lambda,\Lambda])-\sum_{(x,y)\in S_l}\underset{a,b\sim \mathcal{X}}{\Pr}\left(\dist(a,b)\in[\lambda,\Lambda]\right)\right| >\frac{m\cdot \gamma^2\varepsilon_0^2}{4\log^2(\Lambda/\lambda)}\right)\\
\leq~&2\exp\left(-\frac{\frac{1}{32} \cdot m^2\cdot \gamma^4\varepsilon_0^4/\log^4(\Lambda/\lambda)}{m/2+m\cdot \gamma^2\varepsilon_0/\log^2(\Lambda/\lambda)\cdot 1/48}\right)\\
\leq~&2\exp\left(-\frac{\frac{1}{32} \cdot m^2\cdot \gamma^4\varepsilon_0^4/\log^4(\Lambda/\lambda)}{m/2+m/2}\right)\\
=~&2\exp\left(-\frac{1}{32} \cdot m\cdot \gamma^4\varepsilon_0^4/\log^4(\Lambda/\lambda)\right)\\
\leq~& \frac{\delta}{2}\cdot \frac{1}{2(m-1)|I|}\\
\leq~& \frac{\delta}{2}\cdot \frac{1}{2(m-1)},
\end{align*}
where the first inequality follows by plugging $|S_l|=m/2$ i.i.d. random
variables $\mathbf{1}(\dist(x,y)\in [\lambda,\Lambda])$ for all $(x,y)\in S_l,$
$t=(m\cdot \gamma^2\varepsilon_0^2)/(4\log^2(\Lambda/\lambda))$ and $M=1$ into
Lemma~\ref{lem:bernstein}, the second inequality follows by
$\gamma^2\varepsilon_0^2/\log^2(\Lambda/\lambda)\leq 1$, where $\gamma=\Pr_{a,b\sim\mathcal{X}}(\dist(a,b)\in [\lambda,\Lambda]).$  The third inequality
follows by the choice of $m$ and $(m-1)\leq m,|I|\leq 2\log(\Lambda/\lambda)/\varepsilon_0.$
By taking union bound over all the sets $S_1,S_2,\cdots,S_{2(m-1)},$ with probability at least $1-\delta/2,$ we have $\forall l\in[2(m-1)],$
\begin{align*}
\left| \frac{\sum_{(x,y)\in S_l}\mathbf{1}(\dist(x,y)\in [\lambda,\Lambda])}{m/2}-\underset{a,b\sim \mathcal{X}}{\Pr}\left(\dist(a,b)\in[\lambda,\Lambda]\right)\right| \leq\frac{\gamma^2\varepsilon_0^2}{8\log^2(\Lambda/\lambda)}.
\end{align*}
In this case, we have:
\begin{align*}
\small
\left| \sum_{l=1}^{2(m-1)}\sum_{(x,y)\in S_l}\frac{\mathbf{1}(\dist(x,y)\in [\lambda,\Lambda])}{m/2}-2(m-1)\underset{a,b\sim \mathcal{X}}{\Pr}\left(\dist(a,b)\in[\lambda,\Lambda)\right)\right| \leq\frac{2(m-1)\gamma^2\varepsilon_0^2}{8\log^2(\Lambda/\lambda)}.
\end{align*}
Since $S=\bigcup_{l=1}^{2(m-1)}S_l=\{(y_j,y_k)\mid j,k\in[m],j\not=k\},$ we have
\begin{align}
\left|\frac{\sum_{j\not=k}\mathbf{1}(\dist(y_j,y_k)\in [\lambda,\Lambda])}{m(m-1)}-\underset{a,b\sim \mathcal{X}}{\Pr}\left(\dist(a,b)\in[\lambda,\Lambda)\right)\right|\leq \frac{\gamma^2\varepsilon_0^2}{8\log^2(\Lambda/\lambda)}.\label{eq:ub_2}
\end{align}
Thus now, we also obtain an upper bound for the Equation~\eqref{eq:bound_2}.

By taking union bound, we have that with probability at least $1-\delta,$ Equation~\eqref{eq:ub_1} holds for all $i\in I,$ and at the same time, Equation~\eqref{eq:ub_2} holds. In the following, we condition on that Equation~\eqref{eq:ub_1} holds for all $i\in I,$ and Equation~\eqref{eq:ub_2} also holds.

$\forall i\in I,$ we have
\begin{align}
\small
&\Pr_{p'\sim\mathcal{P'}}(p'\in [i,i+1)\cdot \log\alpha)\notag\\
=&\frac{1/(m(m-1))\cdot\sum_{j\not=k}\mathbf{1}(\dist(y_j,y_k)\in [\alpha^i,\alpha^{i+1}))}{1/(m(m-1))\cdot \sum_{j\not =k}\mathbf{1}(\dist(y_j,y_k)\in[\lambda,\Lambda])}\notag\\
\leq~& \frac{\underset{a,b\sim \mathcal{X}}{\Pr}\left(\dist(a,b)\in[\alpha^i,\alpha^{i+1})\right)+\gamma\varepsilon_0^2/(8\log^2(\Lambda/\lambda))}{\gamma-\gamma^2\varepsilon_0^2/(8\log^2(\Lambda/\lambda))}\notag\\
\leq~& \frac{\underset{a,b\sim \mathcal{X}}{\Pr}\left(\dist(a,b)\in[\alpha^i,\alpha^{i+1})\right)}{\gamma-\gamma^2\varepsilon_0^2/(8\log^2(\Lambda/\lambda))}+\frac{\varepsilon_0^2}{4\log^2(\Lambda/\lambda)}\notag\\
\leq~& \frac{\underset{a,b\sim \mathcal{X}}{\Pr}\left(\dist(a,b)\in[\alpha^i,\alpha^{i+1})\right)(1+\gamma\varepsilon_0^2/(4\log^2(\Lambda/\lambda)))}{\gamma}+\frac{\varepsilon_0^2}{4\log^2(\Lambda/\lambda)}\notag\\
\leq~& \frac{\underset{a,b\sim \mathcal{X}}{\Pr}\left(\dist(a,b)\in[\alpha^i,\alpha^{i+1})\right)}{\underset{a,b\sim \mathcal{X}}{\Pr}\left(\dist(a,b)\in[\lambda,\Lambda)\right)}+\frac{\varepsilon_0^2}{2\log^2(\Lambda/\lambda)}\notag\\
=~& \Pr_{p\sim\mathcal{P}}(p\in [i,i+1)\cdot \log\alpha)+\varepsilon_0^2/(2\log^2(\Lambda/\lambda))\label{eq:final_ing_1}
\end{align}
where the first inequality follows by Equation~\eqref{eq:ub_1} and Equation~\eqref{eq:ub_2}, the second inequality follows by $\gamma-\gamma^2\varepsilon_0^2/(8\log^2(\Lambda/\lambda))>\gamma/2,$ the third inequality follows by $1/(1-\eta)\leq (1+2\eta)$ for all $\eta\leq 1/2$ and the last inequality follows by the definition of $\gamma$ and probability is always at most $1$.

Similarly, $\forall i\in I,$ we also have
\begin{align}
\small
&\Pr_{p'\sim\mathcal{P'}}(p'\in [i,i+1)\cdot \log\alpha)\notag\\
=&\frac{1/(m(m-1))\cdot\sum_{j\not=k}\mathbf{1}(\dist(y_j,y_k)\in [\alpha^i,\alpha^{i+1}))}{1/(m(m-1))\cdot \sum_{j\not =k}\mathbf{1}(\dist(y_j,y_k)\in[\lambda,\Lambda])}\notag\\
\geq~& \frac{\underset{a,b\sim \mathcal{X}}{\Pr}\left(\dist(a,b)\in[\alpha^i,\alpha^{i+1})\right)-\gamma\varepsilon_0^2/(8\log^2(\Lambda/\lambda))}{\gamma+\gamma^2\varepsilon_0^2/(8\log^2(\Lambda/\lambda))}\notag\\
\geq~& \frac{\underset{a,b\sim \mathcal{X}}{\Pr}\left(\dist(a,b)\in[\alpha^i,\alpha^{i+1})\right)}{\gamma+\gamma^2\varepsilon_0^2/(8\log^2(\Lambda/\lambda))}-\frac{\varepsilon_0^2}{4\log^2(\Lambda/\lambda)}\notag\\
\geq~& \frac{\underset{a,b\sim \mathcal{X}}{\Pr}\left(\dist(a,b)\in[\alpha^i,\alpha^{i+1})\right)(1-\gamma\varepsilon_0^2/(8\log^2(\Lambda/\lambda)))}{\gamma}-\frac{\varepsilon_0^2}{4\log^2(\Lambda/\lambda)}\notag\\
\geq~& \frac{\underset{a,b\sim \mathcal{X}}{\Pr}\left(\dist(a,b)\in[\alpha^i,\alpha^{i+1})\right)}{\underset{a,b\sim \mathcal{X}}{\Pr}\left(\dist(a,b)\in[\lambda,\Lambda)\right)}-\frac{\varepsilon_0^2}{2\log^2(\Lambda/\lambda)}\notag\\
=~& \Pr_{p\sim\mathcal{P}}(p\in [i,i+1)\cdot \log\alpha)-\varepsilon_0^2/(2\log^2(\Lambda/\lambda))\label{eq:final_ing_2}
\end{align}
where the first inequality follows by Equation~\eqref{eq:ub_1} and Equation~\eqref{eq:ub_2}, the second inequality follows by $\gamma+\gamma^2\varepsilon_0^2/(8\log^2(\Lambda/\lambda))>\gamma,$ the third inequality follows by $1/(1+\eta)\geq (1-\eta)$ for all $\eta\geq0$ and the last inequality follows by the definition of $\gamma$ and probability is always at most $1$.

By combining Equation~\eqref{eq:final_ing_1}, Equation~\eqref{eq:final_ing_2} with Equation~\ref{eq:final}, we complete the proof.
\end{proof}

\subsection{Proof of Theorem~\ref{thm:main_distance_of_distance_distribution}}\label{sec:thm5_proof}

To prove Theorem~\ref{thm:main_distance_of_distance_distribution}, we prove the following theorem first.

\begin{theorem}\label{thm:bourgain_distribution}
Consider a metric space $(\M,\dist).$
Let $y_1,y_2,\cdots,y_m\in\M.$
Let $\mathcal{U}$ be a uniform distribution over multiset $Y=\{y_1,y_2,\cdots,y_m\}.$
Let $\lambda,\Lambda$ be two parameters such that $0<2\lambda\leq \Lambda.$
Let $\mathcal{P}'$ denote LPDD of $\mathcal{U}.$
There exist a mapping $f:X\rightarrow\mathbb{R}^l$ for some $l=O(\log m)$ such that $W(\mathcal{P'},\hat{\mathcal{P}})\leq O(\log\log m),$
where $\hat{\mathcal{P}}$ denotes LPDD of the uniform distribution on the multiset $F=\{f(x_1),f(x_2),\dots,f(x_m)\}\subset\mathbb{R}^l.$
\end{theorem}
\begin{proof}
According to Corollary~\ref{cor:bourgain}, there exists a mapping $f:X\rightarrow \mathbb{R}^l$ for some $l=O(\log m)$ such that $\forall i,j\in[m],\dist(y_i,y_j)\leq \|f(y_i)-f(y_j)\|_2\leq O(\log m)\cdot \dist(y_i,y_j).$
Notice that since $(\M,\dist)$ is a metric space and $f$ holds the above condition, for any $x,y\in \M,$ $\dist(x,y)=\|f(x)-f(y)\|_2=0$ if and only if $x=y.$
Let $\mathcal{U}'$ be the uniform distribution over the multiset $F=\{f(x_1),f(x_2),\cdots,f(x_m)\}.$
Thus, $\Pr_{a,b\sim \mathcal{U}}(a\not =b)=\Pr_{a',b'\sim \mathcal{U'}}(a'\not=b').$
Furthermore, we have $\forall y\in Y,$ $\Pr_{P\sim \mathcal{U}}(p=y)=\Pr_{p'\sim \mathcal{U'}}(p'=f^{-1}(y)).$

Thus, $\forall x,y\in Y,x\not=y,$ we have
\begin{align*}
&\Pr_{a,b\sim \mathcal{U}}(a=x,b=y\mid a\not=b)\\
=~&\Pr_{a,b\sim\mathcal{U}}(a=x,b=y)/\Pr_{a,b\sim\mathcal{U}}(a\not=b)\\
=~&\Pr_{a\sim\mathcal{U}}(a=x)\Pr_{b\sim\mathcal{U}}(b=y)/\Pr_{a,b\sim\mathcal{U}}(a\not=b)\\
=~&\Pr_{a'\sim\mathcal{U'}}(f^{-1}(a')=x)\Pr_{b'\sim\mathcal{U'}}(f^{-1}(b')=y)/\Pr_{a',b'\sim\mathcal{U'}}(a'\not=b')\\
=~&\Pr_{a',b'\sim\mathcal{U'}}(f^{-1}(a')=x,f^{-1}(b')=y\mid a'\not=b').
\end{align*}
Then we can conclude that
\begin{align*}
&W(\mathcal{P}',\hat{\mathcal{P}})\\
\leq~&\sum_{x,y\in Y:x\not=y}\Pr_{a,b\sim \mathcal{U}}(a=x,b=y\mid a\not=b)\cdot|\log(\dist(x,y))-\log(\|f(x)-f(y)\|_2)|\\
=~&\sum_{x,y\in Y:x\not=y}\Pr_{a,b\sim \mathcal{U}}(a=x,b=y\mid a\not=b)\cdot\left|\log\left(\frac{\dist(x,y)}{\|f(x)-f(y)\|_2}\right)\right|\\
\leq~&\sum_{x,y\in Y:x\not=y}\Pr_{a,b\sim \mathcal{U}}(a=x,b=y\mid a\not=b)\cdot O(\log\log m)\\
=~&O(\log\log m).
\end{align*}
\end{proof}

In the following, we formally state the complete version of Theorem~\ref{thm:main_distance_of_distance_distribution}.
\begin{theorem}
Consider a universe of the data $\M$ and a distance function $\dist:\M\times\M\rightarrow \mathbb{R}_{\geq 0}$ such that $(\M,\dist)$ is a metric space.
Let $\mathcal{X}$ be a data distribution over $\M$ which satisfies $\Pr_{a,b\sim \mathcal{X}}(a\not=b)\geq 1/2.$
Let $X$ be a multiset which contains $n$ i.i.d. observations $x_1,x_2,\cdots,x_n\in\M$ generated from the data distribution $\mathcal{X}.$
Let $\lambda=\min_{i\in[n/2-1]:x_{i}\not=x_{i+1}}\dist(x_{i},x_{i+1}),$ and $\Lambda=\max(\max_{i\in[n/2-1]}\dist(x_{i},x_{i+1}),2\lambda).$
Let $\mathcal{P}$ be the $(\lambda,\Lambda)-$LPDD of the original data distribution $\mathcal{X}.$
If $n\geq \log^c_0(\Lambda/\lambda)$ for a sufficiently large constant $c_0$, then with probability at least $0.99,$ we can find a distribution $\mathcal{F}$ on $F\subset\mathbb{R}^l$ for $l=O\left(\log\log(\Lambda/\lambda)\right),|F|\leq C\log^4(\Lambda/\lambda)\log(\log(\Lambda/\lambda))$ where $C$ is a sufficiently large constant, such that $W(\mathcal{P},\hat{\mathcal{P}})\leq O(\log\log\log (\Lambda/\lambda)),$ where $\hat{\mathcal{P}}$ is the LPDD of distribution $\mathcal{F}$
\end{theorem}
\begin{proof}
We describe how to construct the distribution $\mathcal{F}.$
Let $\lambda=\min_{i\in[n/2-1]:x_{i}\not=x_{i+1}}\dist(x_{i},x_{i+1}),$ and $\Lambda=\max(\max_{i\in[n/2-1]}\dist(x_{i},x_{i+1}),2\lambda).$
By applying Theorem~\ref{thm:estimate_range_of_lambda}, with probability at least $0.999,$ we have
\begin{align}
\Pr_{a,b\sim \mathcal{X}} (\dist(a,b)\in[\lambda,\Lambda])\geq \Pr_{a,b\sim\mathcal{X}}(a\not=b)-1/\Omega(n).\label{eq:high_prob_in_range}
\end{align}
Let the above event be $\mathcal{E}_1.$
In the remaining of the proof, let us condition on $\mathcal{E}_1.$

Let $m=C\log^4(\Lambda/\lambda)\log(\log(\Lambda/\lambda))$ where $C$ is a sufficiently large constant.
Let $Y=\{x_{n/2+1},x_{n/2+2},\cdots,x_{n/2+m}\}.$
Let $\mathcal{P'}$ be the $(\lambda,\Lambda)-$LPDD of the uniform distribution on $Y$.
Notice that Equation~\eqref{eq:high_prob_in_range} implies $\Pr_{p\sim \mathcal{P'}}(p\in[\lambda,\Lambda])\geq 1/4.$
Then, according to Theorem~\ref{thm:small_set_good_approx}, with probability at least $0.999,$ we have
\begin{align}
W(\mathcal{P},\mathcal{P'})\leq 1.\label{eq:need_to_add_tri}
\end{align}
Let the above event be $\mathcal{E}_2.$
In the remaining of the proof, let us condition on $\mathcal{E}_2.$

Equation~\eqref{eq:high_prob_in_range} also implies the following thing:
\begin{align*}
\Pr_{a,b\sim \mathcal{X}} (\dist(a,b)\in[\lambda,\Lambda]\mid a\not=b)\geq 1-1/(\Omega(n)\cdot \Pr_{a,b\sim \mathcal{X}}(a\not=b))\geq 1-1/\poly(\log(\Lambda/\lambda)).
\end{align*}
By taking union bound over all $i,j\in\{n/2+1,n/2+2,\cdots,n/2+m\},i\not=j,$ with probability at least $0.999,$ we have either $x_i=x_j$ or $\dist(x_i,x_j)\in[\lambda,\Lambda].$
Let the above event be $\mathcal{E}_3.$
In the remaining of the proof, let us condition on $\mathcal{E}_3.$

Due to $\mathcal{E}_3,$ we can just regard $\mathcal{P'}$ as the LPDD of the uniform distribution on $Y$. 
Then, by applying Theorem~\ref{thm:bourgain_distribution}, we can construct a uniform distribution $\mathcal{F}$ on $F\subset\mathbb{R}^{l}$ where $|F|\leq m.$
Let $\hat{\mathcal{P}}$ be the LPDD of $\mathcal{F}.$
According to the Theorem~\ref{thm:bourgain_distribution}, we have $W(\mathcal{P}',\hat{\mathcal{P}})\leq O(\log\log m)\leq O(\log\log\log(\Lambda/\lambda)).$
Then by combining with Equation~\eqref{eq:need_to_add_tri}, we have $W(\mathcal{P},\hat{\mathcal{P}})\leq W(\mathcal{P},\mathcal{P}')+W(\mathcal{P}',\hat{\mathcal{P}})\leq 1+O(\log\log\log(\Lambda/\lambda))=O(\log\log\log(\Lambda/\lambda)).$
Thus, we complete the proof.

By taking union bound over $\mathcal{E}_1,\mathcal{E}_2,\mathcal{E}_3,$ the success probability is at least $0.99$.
\end{proof}

\end{document}